\documentclass[preprint,12pt]{elsarticle}

\usepackage{amssymb}
\usepackage{amsmath}
\usepackage{amsthm}
\usepackage{subcaption}
\usepackage[utf8]{inputenc} % allow utf-8 input
\usepackage[T1]{fontenc}    % use 8-bit T1 fonts
\usepackage{hyperref}       % hyperlinks
\usepackage{url}            % simple URL typesetting
\usepackage{booktabs}       % professional-quality tables
\usepackage{amsfonts}       % blackboard math symbols
\usepackage{nicefrac}       % compact symbols for 1/2, etc.
\usepackage{microtype}      % microtypography
\usepackage{lipsum}
\usepackage{fancyhdr}       % header
\usepackage{graphicx}       % graphics
\usepackage{algorithm, algpseudocode}
\usepackage{enumitem}
\usepackage{xcolor}
\usepackage{multirow}
\usepackage[flushleft]{threeparttable}
\newtheorem{remark}{Remark}
\newtheorem{theorem}{Theorem}
\newtheorem{lemma}{Lemma}

\newtheorem{definition}{Definition}
\DeclareMathOperator*{\argmax}{argmax}
\DeclareMathOperator*{\argmin}{argmin}
\let\bs\boldsymbol
\journal{Computers \& Chemical Engineering}

\begin{document}

\begin{frontmatter}



\title{BONSAI: Structure-Exploiting Robust Bayesian Optimization for Networked Black-Box Systems under Uncertainty} %% Article title


\author{Akshay Kudva and Joel A. Paulson}

\affiliation{organization={The Ohio State University},%Department and Organization
            addressline={Department of Chemical and Biomolecular Engineering}, 
            city={Columbus},
            postcode={43210}, 
            state={OH},
            country={USA}}

%% Abstract
\begin{abstract}
Optimal design under uncertainty remains a fundamental challenge in advancing reliable, next-generation process systems. Robust optimization (RO) offers a principled approach by safeguarding against worst-case scenarios across a range of uncertain parameters. However, traditional RO methods typically require known problem structure, which limits their applicability to high-fidelity simulation environments. To overcome these limitations, recent work has explored robust Bayesian optimization (RBO) as a flexible alternative that can accommodate expensive, black-box objectives. Existing RBO methods, however, generally ignore available structural information and struggle to scale to high-dimensional settings. In this work, we introduce BONSAI (Bayesian Optimization of Network Systems under uncertAInty), a new RBO framework that leverages partial structural knowledge commonly available in simulation-based models. Instead of treating the objective as a monolithic black box, BONSAI represents it as a directed graph of interconnected white- and black-box components, allowing the algorithm to utilize intermediate information within the optimization process. We further propose a scalable Thompson sampling-based acquisition function tailored to the structured RO setting, which can be efficiently optimized using gradient-based methods. We evaluate BONSAI across a diverse set of synthetic and real-world case studies, including applications in process systems engineering. Compared to existing simulation-based RO algorithms, BONSAI consistently delivers more sample-efficient and higher-quality robust solutions, highlighting its practical advantages for uncertainty-aware design in complex engineering systems.
\end{abstract}


% \begin{highlights}
% \end{highlights}

%% Keywords

\end{frontmatter}



\section{Introduction} 
\label{sec:introduction}

Designing engineered systems under uncertainty is a fundamental challenge in real-world applications. Uncertainty, which can arise from sources such as manufacturing imperfections, environmental variability, or modeling approximations, can significantly degrade the performance or feasibility of solutions obtained from deterministic (i.e., nominal) optimization. Even relatively small deviations from assumed conditions can cause nominally optimal solutions to perform poorly or even fail in practice.
To address this issue, robust optimization (RO) frameworks~\cite{blankenship1976infinitely, Beyer2007, BenTal2009, gabrel2014recent} have been developed to seek solutions that remain effective even under worst-case or uncertain scenarios. While conceptually powerful, solving RO problems is notoriously difficult due to their bilevel nature, which involves an outer-loop optimization over design variables and an inner-loop optimization to find worst-case uncertainty realizations. In certain special cases -- where the full algebraic structure is known and satisfies convexity/concavity conditions -- global optimization is tractable. However, such assumptions often fail to hold in practice.

In many engineering problems, system behavior can be accurately captured by high-fidelity simulations (or so-called ``digital twins''), which are increasingly used to support data-driven decision-making \cite{villalonga2021decision, nghiem2023physics, iranshahi2025digital}. These simulators may involve proprietary code, computationally intensive physics-based solvers, or stochastic agent-based models. As such, they are typically treated as black boxes: we can observe inputs and outputs, but cannot compute gradients or exploit internal structure. In such settings, RO becomes especially challenging, as each function evaluation is costly, meaning that only a limited number of simulations can be performed to identify robust designs.
This challenge has motivated a growing body of work on simulation-based RO. An early contribution in this area is the method of Bertsimas et al.~\cite{Bertsimas_unconstrained}, which performs robust local optimization using finite-difference estimates of gradients followed by a cone-based robust step calculation. While theoretically and conceptually elegant, this approach requires a large number of function evaluations at each iteration, making it unsuitable for computationally expensive simulators.

More recently, Bayesian optimization (BO)~\cite{Mockus1989, Frazier2018, paulson2024bayesian} has emerged as a powerful framework for black-box global optimization, combining probabilistic modeling with (sample-efficient) experimental design. Several robust BO methods have been proposed that extend this framework to handle uncertainty, including ARBO~\cite{Paulson2022}, which jointly models the objective as a function of both design and uncertain variables. This joint modeling enables informed sampling over the combined design-uncertainty space, in contrast to prior two-stage methods such as MiMReK~\cite{Marzat2015}, which attempt to solve the inner problem at every iteration of the outer problem. Alternative acquisition functions to the confidence bound-based ARBO method have also been explored including lookahead and information-based acquisitions for max-min problems in \cite{MinMaxBO}. There have also been BO methods developed to optimize over stochastic risk measures \cite{RiskMeasuresBO} or tackle distributionally robust optimization problems, e.g., \cite{DristributionallyRobust, Sessa2020, Aldeghi2021}.
While existing robust BO methods have demonstrated success, most treat the objective as a fully black-box function. 
This assumption is often overly conservative in scientific applications where partial knowledge of the system structure is available -- such as known transformations, internal submodels, or hierarchical decompositions. Leveraging this knowledge can substantially improve optimization efficiency.

In the nominal setting, recent work has shown that exploiting such structure can lead to dramatic improvements in sample efficiency. COBALT~\cite{COBALT} models the objective as a composite function, combining known outer transformations with multi-output black-box inner models. BOIS~\cite{gonzalez2025implementation} uses an acquisition function derived by deriving an approximate linearized model that simplifies uncertainty propagation compared to the improvement-based strategy proposed in COBALT. 
The BOFN framework~\cite{BOFN} generalizes these ideas to function networks, where the objective is expressed as a directed acyclic graph (DAG) of known and unknown components. This approach has been further extended to support partial evaluations~\cite{BOFN_partial} and causal inference under intervention-based queries (e.g., the MCBO method \cite{MCBO}). 

These advances have led to up to orders-of-magnitude improvement in sample efficiency for certain complex design problems. However, none of these methods address the robust setting, which we argue is a critical gap. Extending structure-aware BO to RO is non-trivial, as the network-based modeling of the objective leads to non-Gaussian posterior distributions, which complicates the design and optimization of the acquisition function (used to decide the next-best candidate point to query). These difficulties are amplified in the robust setting since we need to deal with both design variables and adversarial uncertainty variables at every iteration. One exception is the DRACO method \cite{DRACO}, which is designed for robust optimization problems that exhibit a special additive black-box structure that allows for closed-form uncertainty propagation (effectively extends ARBO to this structured setting). Although powerful, DRACO only works for a relatively narrow class of functions and thus cannot tackle the more general setting of interest in this work. Another exception, CBO-MW \cite{CBO_MW}, arises from the \textit{causal} BO literature \cite{aglietti2020causal}, which extends MCBO to settings where uncontrollable external interventions act adversarially on a known causal DAG. Rather than posing a nested max-min problem over the design and uncertainty variables, CBO-MW treats the problem as online learning over a finite set of interventions. Although this can achieve no-regret guarantees for acyclic graphs, it does not tackle our more general setting of continuous robust design over cyclic graphs. 

In this work, we present \textbf{BONSAI} (\textbf{B}ayesian \textbf{O}ptimization on \textbf{N}etwork \textbf{S}ystems under uncert\textbf{AI}nty), the first robust BO algorithm (to our knowledge) that leverages structured prior knowledge in the form of a function network. BONSAI represents the objective as a directed graph, where each node can be either a black- or white-box function. This enables it to flexibly leverage intermediate outputs and known transformations during both the learning and optimization process.
Figure~\ref{fig:pse_example} illustrates a representative example of such a structured design setting in process systems engineering. Here, an overall economic objective depends on a number of interconnected modules, including plant design, environmental analysis, and product development via physical experimentation. Each module is represented as a distinct function -- some of which are based on proprietary process simulators or complex physics-based models such as computational fluid dynamics (CFD). These components are coupled through intermediate variables and may depend on both global design variables (e.g., material choices, plant layout) and uncertain parameters (e.g., flow variability). A naive black-box BO approach would ignore these dependencies and treat the entire system as a monolithic mapping from inputs to outputs, leading to potentially inefficient exploration and relatively slow convergence to the desired solution. Conversely, fully equation-oriented optimization is often infeasible since it would be difficult to build complete and accurate models for each and every subsystem. BONSAI provides a principled framework for navigating this middle ground by using Gaussian process (GP) surrogates to model unknown subsystems and combining them with known structural information to guide robust, sample-efficient optimization. Table \ref{tab:method-comparison} provides a high-level comparison of BONSAI to the existing approaches introduced in the previous paragraphs. 

\begin{figure}[tb]
  \centering
  \includegraphics[width=0.95\textwidth]{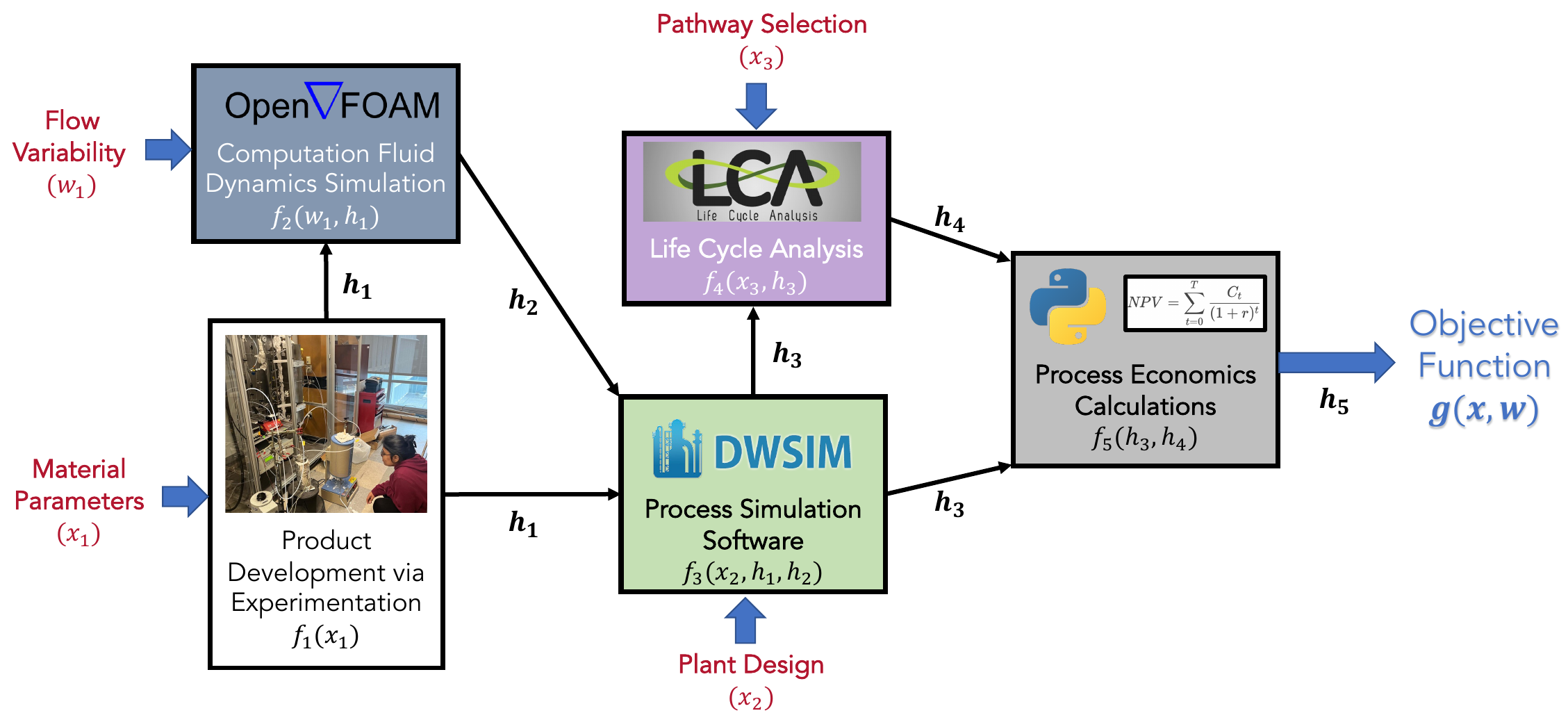}
  \caption{Example function network for robust process systems design. The overall objective function $g(x, w)$ combines outputs from five interconnected modules: CFD simulations, process simulation software, life cycle analysis (LCA), experimentation-based product development, and economic post-processing. Design variables $\bs{x}$ (e.g., plant design choices, pathway selection) and uncertain inputs $\bs{w}$ (e.g., flow variability) propagate through a network of intermediate variables ($h_1$ through $h_5$), capturing modular structure and interactions among subsystems. BONSAI exploits this structure for robust, sample-efficient optimization in the presence of uncertainty.}
  \label{fig:pse_example}
\end{figure}

To summarize, our key contributions are:
\begin{enumerate}
    \item We introduce BONSAI, a framework for robust BO over function networks that handles general combinations of design variables, uncertainties, and known/unknown components.
    \item We develop a scalable, two-stage acquisition function that extends Thompson sampling to the robust function network setting, overcoming the challenge of non-Gaussian posteriors.
    \item BONSAI is, to our knowledge, the first robust BO method capable of operating on cyclic function networks, enabling the modeling of feedback loops and recycle streams commonly encountered in process systems applications.
    \item We establish the first finite-time regret bounds for BO over function networks (including cyclic networks) thereby providing theoretical guarantees that connect a special case of BONSAI to established BO theory.
    \item We demonstrate the performance benefits of incorporating uncertainty and leveraging function network structure on a set of synthetic and real-world benchmarks, including cases with varying network topologies and dimensionalities.
\end{enumerate}

The remainder of this paper is organized as follows. Section~\ref{sec:problem-setting} defines the problem of interest. Section~\ref{sec:BONSAI} presents the BONSAI framework, outlines its algorithmic components, and summarizes our theoretical analysis of its performance in the nominal setting. Section~\ref{sec:numerical-experiments} reports numerical results across several benchmark tasks. Finally, Section~\ref{sec:conclusion} concludes with a discussion of implications and some potential future directions.

\begin{table}[t]
\centering
\begin{threeparttable}
\setlength{\tabcolsep}{5pt}
\caption{Comparison of the proposed BONSAI approach to related robust and structure-exploiting Bayesian optimization (BO) methods.} 
\label{tab:method-comparison}
\begin{tabular}{l c c c c}
\hline
Method & Structure$^{\dagger}$ & Robust (max-min) & Cycles? \\
\hline
\textbf{BONSAI (this work)} & FN-Cyclic & Yes & Yes \\
ARBO~\cite{Paulson2022,Stableopt} & BB & Yes & n/a \\
StableOpt~\cite{Stableopt} & BB & Yes & n/a \\
MiMaReK~\cite{Marzat2015} & BB & Yes & n/a \\
Weichert \& Kister~\cite{MinMaxBO} & BB & Yes & n/a \\
DRACO~\cite{DRACO} & Composite & Yes & No \\
COBALT~\cite{COBALT} & Composite & No & No \\
BOFN~\cite{BOFN} & FN-DAG & No & No \\
BOIS~\cite{gonzalez2025implementation} & FN-Cyclic & No & Yes \\
% BOFN (partial evals)~\cite{BOFN_partial} & FN--DAG & No & No \\
MCBO~\cite{MCBO} & Causal-DAG & No & No \\
CBO-MW~\cite{CBO_MW} & Causal-DAG & Yes & No \\
\hline
\end{tabular}
\begin{tablenotes}
    \item[$\dagger$] 
    \footnotesize
    Structure abbreviations: BB (black-box; no explicit structure), Composite ($g\circ h$ with $g$ known and $h$ black-box), FN-DAG (function network expressed as a directed acyclic graph), Causal-DAG (directed acyclic graph considering causal interventions), FN-Cyclic (general function network that can include cycles).
\end{tablenotes}
\end{threeparttable}
\end{table}
 
\section{Problem Description}
\label{sec:problem-setting}

In this work, we are interested in solving the following max-min (robust) optimization problem involving an expensive-to-evaluate function
\begin{align} \label{eq:max-min}
    \max_{\bs{x} \in \mathcal{X}} \min_{\bs{w} \in \mathcal{W}} ~ g(\bs{x}, \bs{w}),
\end{align}
where $\bs{x} \in \mathcal{X} \subset \mathbb{R}^{n_x}$ are the design variables constrained to the compact set $\mathcal{X}$, $\bs{w} \in \mathcal{W} \subset \mathbb{R}^{n_w}$ are the uncertain parameters constrained to the compact set $\mathcal{W}$, and $g : \mathcal{X} \times \mathcal{W} \to \mathbb{R}$ is the objective (or performance) function. The solution to \eqref{eq:max-min}, which we denote by $\bs{x}^\star$, is the design that maximizes the worst-case objective value across all feasible inputs, i.e., $\min_{\bs{w} \in \mathcal{W}} g(\bs{x}^\star, \bs{w}) \geq \min_{\bs{w} \in \mathcal{W}} g(\bs{x}, \bs{w})$ for all $\bs{x} \in \mathcal{X}$. This formulation can be interpreted as a sequential two-player game: the design $\bs{x}$ must be selected first, while the uncertainty $\bs{w}$ responds adversarially to degrade performance.

The choice of algorithm for solving \eqref{eq:max-min} depends heavily on the properties of the function $g$. Here, we consider a flexible and structured representation in which $g$ is defined implicitly through a directed graph of component/base functions $f_1, f_2, \ldots, f_K$, each corresponding to a node in the graph. These functions may be expensive to evaluate and depend on subsets of the design variables $\bs{x}$, uncertain parameters $\bs{w}$, and outputs of the other nodes. Let $V = \{1, 2, \ldots, K\}$ denote the set of nodes and let $E = \{ (j, k) : f_k \text{ depends on } f_j \}$ define the directed edges that encode functional dependencies. Each node function is defined by
\begin{align} \label{eq:network-functions}
    h_k = f_k\left( \bs{x}_{I_x(k)}, \bs{w}_{I_w(k)}, \{ h_j \}_{j \in J(k)} \right), \qquad k = 1, \ldots, K,
\end{align}
where $\bs{x}_{I_x(k)}$ and $\bs{w}_{I_w(k)}$ are the subsets of design and uncertain variables relevant to node $k$ and $J(k) \subseteq \{1,\ldots,K\} \setminus \{k\}$ indexes the nodes of the other functions whose outputs serve as input to $f_k$. Throughout the paper, we will often refer to \eqref{eq:network-functions} as a ``function network''.

For notational simplicity, we define $\bs{H} = [h_1, h_2, \ldots, h_K]^\top$ as the vector of node outputs and introduce the vector-valued function $\bs{F}( \bs{x}, \bs{w}, \bs{H} )$ as the concatenation of the right-hand sides of \eqref{eq:network-functions}. This allows us to express the system compactly as follows
\begin{align} \label{eq:compact-network}
    \bs{H} = \bs{F}( \bs{x}, \bs{w}, \bs{H} ).
\end{align}
This formulation generalizes prior work on acyclic function networks (e.g., \cite{BOFN, BOFN_PE}) in two important ways: (i) it explicitly models the effect of uncertainty and (ii) it allows for arbitrary directed graph structures, including cycles arising from feedback loops or recycle streams -- features that are common in process systems engineering applications.

We assume the final objective $g(\bs{x}, \bs{w})$ is a known linear transformation of the vector of node outputs
\begin{align} \label{eq:network-objective}
    g(\bs{x}, \bs{w}) = \bs{c}^\top \bs{H}^\star( \bs{x}, \bs{w} ),
\end{align}
where $\bs{c} \in \mathbb{R}^K$ is a projection vector and $\bs{H}^\star( \bs{x}, \bs{w} )$ denotes the solution to \eqref{eq:compact-network} for fixed $(\bs{x}, \bs{w})$. This is without loss of generality, as nonlinear transformations of the outputs can be included as additional nodes in the graph, with $\bs{c}$ being appropriately defined to select the corresponding final node. 
We further assume that the solution to \eqref{eq:compact-network} is unique, which is satisfied under mild regularity conditions such as when the function $\bs{F}( \bs{x}, \bs{w}, \bs{H} )$ is a contraction in $\bs{H}$ (e.g., by the Banach fixed-point theorem; see \cite[Chapter 3]{OrtegaReinboldt1970}). These conditions are often satisfied in process models with feedback or recycle streams, where the internal state of the system evolves in a stable and consistent way.
Lastly, a key modeling assumption in this work is that we have the ability to query the node functions $f_k$ at arbitrary inputs $(\bs{x}_{I_x(k)}, \bs{w}_{I_w(k)}, \bs{H}_{J(k)})$. This assumption is satisfied in many simulation-based engineering applications, where internal state variables and intermediate results are exposed during a model run. Examples include flowsheet simulators, physics-based solvers, and digital twins composed of coupled subsystem models. In such settings, it is typically straightforward to extract the outputs of intermediate components through, e.g., solver callbacks.

Previous work, such as ARBO \cite{Paulson2022}, attempt to solve \eqref{eq:max-min} using a robust Bayesian optimization perspective under a Gaussian process (GP) prior placed directly on $g$. However, such models are agnostic to the underlying structure in \eqref{eq:compact-network}-\eqref{eq:network-objective}, resulting in losses in sample efficiency. In this work, we instead place GP priors directly on the individual node functions $\{f_k\}_{k=1}^K$ (which is more natural given the previously defined representation) and use the graph structure to propagate uncertainty and guide sampling. By leveraging this additional structural information, we aim to significantly improve the sample efficiency of robust optimization in expensive simulation-based systems.

\section{The Proposed BONSAI Framework} \label{sec:BONSAI}

In this section, we introduce our proposed approach for incorporating the function network structure into the robust (max-min) optimization \eqref{eq:max-min}. We first describe the probabilistic model that we use to represent the unknown (black-box) node functions $f_1,\ldots,f_K$, which yields a fundamentally different model for the overall objective function $g$ than existing approaches. We then present the complete BONSAI algorithm as well as establish some theoretical convergence results (for a special case) and discuss several important practical implementation details.

\subsection{Probabilistic surrogate model}

We model the set of $K$ node functions $f_1, \ldots, f_K$ as being drawn from independent Gaussian process (GP) priors. To simplify notation, we let $\bs{z}_k = ( \bs{x}_{I_x(k)}, \bs{w}_{I_w(k)}, \bs{H}_{J(k)} )$ denote the full set of inputs to $f_k$ and, with slight abuse of notation, write $f_k(\bs{z}_k)$ interchangeably with $f_k(\bs{x}_{I_x(k)}, \bs{w}_{I_w(k)}, \bs{H}_{J(k)})$. Each node function can be mathematically expressed as
\begin{align}
    f_k \sim \mathcal{GP}( \mu_{0,k}, \Sigma_{0,k} ), \qquad k = 1, \ldots, K,
\end{align}
where $\mu_{0,k} : \mathbb{R}^{n_{z,k}} \to \mathbb{R}$ is the prior mean function and $\Sigma_{0,k} : \mathbb{R}^{n_{z,k}} \times \mathbb{R}^{n_{z,k}} \to \mathbb{R}$ is the prior covariance (or kernel) function. Here, $n_{z,k}$ denotes the dimension of the input vector $\bs{z}_k$. The mean function captures any strong prior belief about the expected behavior of $f_k$, while the covariance function encodes assumptions about the smoothness, amplitude, and variability of the function across the input space. In particular, it defines how correlated the outputs are for two different input values -- points that are closer in input space are typically assumed to yield more similar outputs. A more detailed discussion on covariance functions and their role in shaping GP behavior can be found in \cite[Chapter 4]{Rasmussen2006}. It is worth noting that we have complete freedom in setting the prior mean and covariance functions. Thus, for any white-box function $f_k$ (whose structure is fully known), we can simply set $\mu_{0,k}(\bs{z}_k) = f_k( \bs{z}_k )$ and $\Sigma_{0,k}( \bs{z}_k, \bs{z}_k' ) = 0$ for all $\bs{z}_k$ and $\bs{z}_k'$ to exactly encode this knowledge. 

Let $y_{l,k} = f_k(\bs{z}_{l,k}) + \epsilon_{l,k}$ denote the observed value at the $l$th sample, where $\bs{z}_{l,k}$ is the corresponding input to node $k$ and $\epsilon_{l,k}$ is additive observation noise. In this work we assume $\epsilon_{l,k} \overset{\text{i.i.d.}}{\sim} \mathcal{N}(0,\sigma_k^2)$, although more general noise models (e.g., heteroscedastic or heavy-tailed) could also be incorporated. This formulation naturally covers the case of noise-free evaluations by setting $\sigma_{k}^2 = 0$. Let $\mathcal{D}_{n,k} = \{ (\bs{z}_{l,k}, y_{l,k}) \}_{l=1}^n$ denote the dataset of observations for the $k$th node.
A key property of GPs is that they are closed under conditioning: given data $\mathcal{D}_{n,k}$, the posterior distribution of $f_k$ remains a GP with updated mean and covariance functions $\mu_{n,k}$ and $\Sigma_{n,k}$. Because we assume independence across nodes and have access to all intermediate outputs, each node’s GP can be updated independently. The posterior mean and covariance functions for each $f_k$ are thus given by
\begin{subequations} \label{eq:posterior-gp}
\begin{align}
    \mu_{n,k}(\bs{z}_k) &= \mu_{0,k}( \bs{z}_k ) + \Sigma_{0,k}( \bs{z}_k, \bs{z}_{1:n,k} ) \bs{K}_{n, k}^{-1} ( y_{1:n,k} - \mu_{0,k}( \bs{z}_{1:n,k} ) ), \\
    \Sigma_{n,k}( \bs{z}_k, \bs{z}_k' ) &= \Sigma_{0,k}( \bs{z}_k, \bs{z}_k') - \Sigma_{0,k}( \bs{z}_k, \bs{z}_{1:n,k} ) \bs{K}_{n, k}^{-1} \Sigma_{0,k}( \bs{z}_{1:n,k}, \bs{z}_k' ),
\end{align}
\end{subequations}
where we use standard operator overloading to apply the covariance function element-wise to the vector of inputs $\bs{z}_{1:n,k}$ and $\bs{K}_{n, k} = \Sigma_{0,k}(\bs{z}_{1:n,k}, \bs{z}_{1:n,k}) + \sigma^2_k \bs{I}$.

The resulting posterior over $\{f_k\}_{k=1}^K$ defines an implicit posterior over the objective $g(\bs{x}, \bs{w})$ via the system in \eqref{eq:compact-network}-\eqref{eq:network-objective}. However, this posterior is generally non-Gaussian due to nonlinear composition and potential cycles in the function network, which define $\bs{H}^\star(\bs{x}, \bs{w})$ via a stochastic fixed-point equation. In principle, one could approximate this distribution using sampling-based methods such as Markov chain Monte Carlo (MCMC) \cite{jones2022markov}, but doing so within a sequential optimization framework is computationally expensive, as it requires generating posterior samples at many candidate inputs. In the next section, we develop an efficient strategy to circumvent this challenge. 

To build intuition for this modeling paradigm, we illustrate its impact on a simple example problem that depends on a single design variable $x$ and single uncertainty $w$ in Figure~\ref{fig:grey_box}. The left column shows the exact (white-box) function $g(x, w)$ and its (implicitly-defined) projected objective function $G(x) = \min_{w \in \mathcal{W}} g(x, w)$. In the middle, we fit a standard black-box GP model directly to $g$ using a fixed number of samples (black dots). While the black-box model captures broad trends, it exhibits poor uncertainty quantification and misses fine-scale features, resulting in inaccurate estimation of $G(x)$. On the right, we show a ``grey-box'' GP model constructed by modeling $g$ via a three-node function network. Despite being trained on the same data, the grey-box model yields substantially improved predictions and better-calibrated confidence intervals. This improvement arises from its ability to encode structural prior knowledge, which constrains and guides the surrogate to more effectively infer the system’s response surface from limited data. 

\begin{figure}[htb!]
  \centering
  \includegraphics[width=0.9\textwidth]{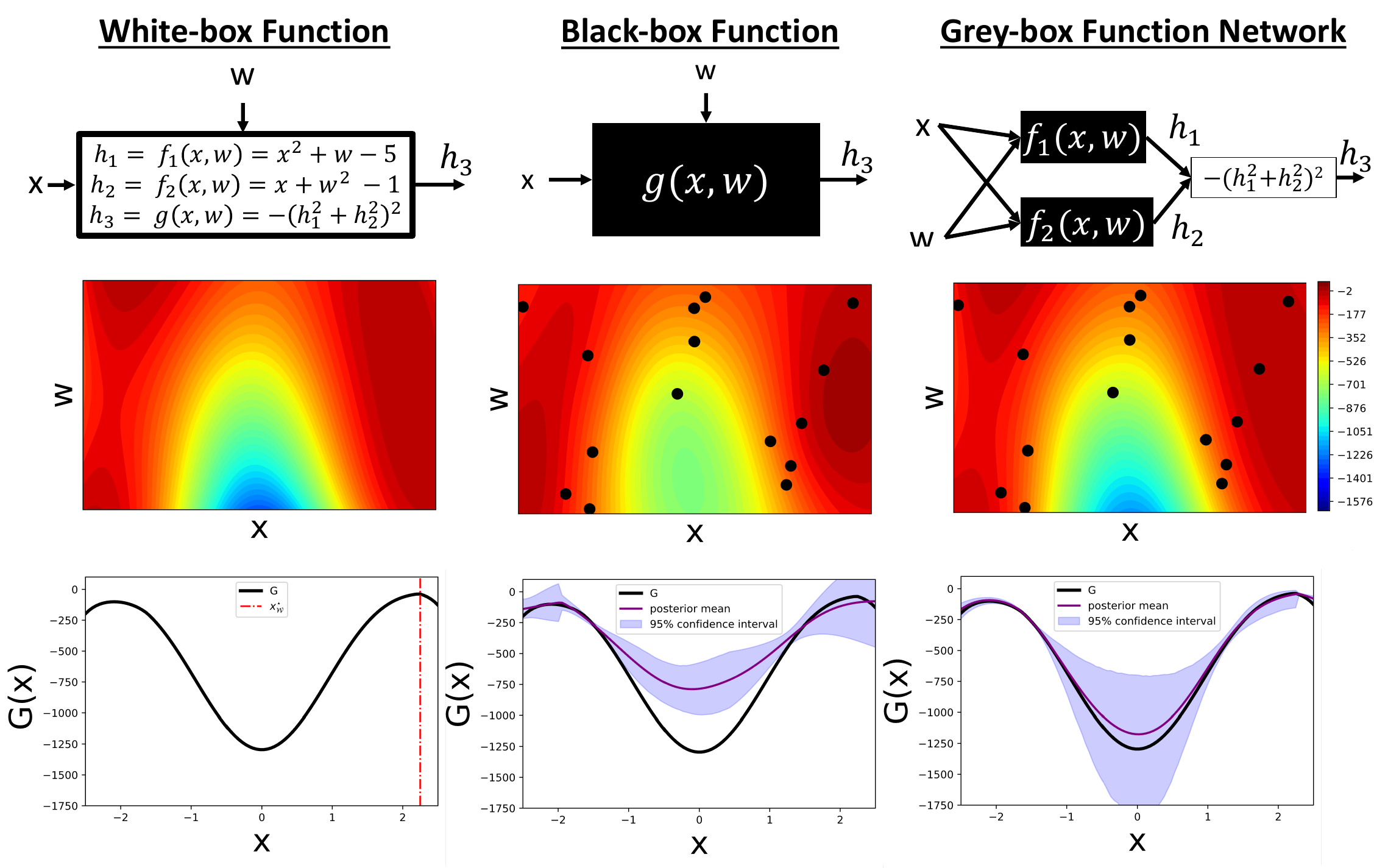}
  \caption{Comparison of white-box, black-box, and the function network (grey-box or hybrid) modeling approaches for the objective function \(g(x, w) = -\left(f_1(x, w)^2 + f_2(x, w)^2\right)^2\) where \(f_1(x, w) = x^2 + w - 5\) and \(f_2(x, w) = x + w^2 - 1\), with \(x \in [-2.5,2.5]\) and \(w \in [-1,1]\). \textbf{Top:} Illustration of function representations used in each approach, reflecting different amounts of prior knowledge. \textbf{Middle:} Contour plots of the true objective (left) and the posterior mean surfaces from the black-box (center) and grey-box (right) models. Black dots indicate sample locations. \textbf{Bottom:} The projected objective function \(G(x) = \min_{w} g(x, w)\), showing the true value (black), posterior mean (violet), and 95\% confidence band (violet cloud). The grey-box model provides substantially improved prediction and uncertainty estimates compared to the black-box baseline. The red vertical line in the white-box plot indicates the true max-min solution.}
  \label{fig:grey_box}
\end{figure}

\subsection{Overview of the BONSAI algorithm}

Given the posterior distributions over the node functions $f_1, \ldots, f_K$ induced by the observed data $\mathcal{D}_{t,1},\ldots,\mathcal{D}_{t,K}$ at any iteration $t$, the core challenge is to decide where to evaluate next in order to efficiently solve the robust optimization problem \eqref{eq:max-min}. A natural but inefficient strategy would be to select samples that maximize the posterior variance of $g$, aiming to reduce global uncertainty in the objective function. While such exploration-based approaches are useful in Bayesian experimental design, they waste evaluations in regions of the domain that do not influence the max-min solution. On the other hand, a purely exploitative strategy that plugs in the posterior mean of $g$ and solves the resulting deterministic max-min problem can be overly greedy, particularly under adversarial uncertainty. It tends to overfit to high-mean regions without adequately exploring the worst-case behavior, often leading to suboptimal or brittle solutions.

A principled BO method must balance exploration and exploitation. However, for the function network structure considered in this work, computing standard acquisition functions is particularly difficult. The posterior over $g$ is not Gaussian due to nonlinear compositions and potential cycles in the network, which define $\bs{H}$ implicitly through a stochastic fixed-point equation. Evaluating the mean or variance of $g$ requires sampling or approximating this distribution, which is computationally intensive and does not scale easily.
To address these challenges, BONSAI employs a two-stage Thompson sampling (TS) strategy \cite{pmlr-v84-kandasamy18a}. TS is a simple yet powerful Bayesian sampling heuristic that selects actions in proportion to their probability of being optimal. For our setting, TS is especially attractive because it avoids the need to explicitly characterize the posterior over $g$. Instead, we sample directly from each node's GP posterior, resulting in a stochastic realization of the full function network. This sample, once drawn, defines a deterministic surrogate model $\widehat{\bs{F}}_t$, and in turn a deterministic realization of the objective, which we denote as $\widehat{g}_t(\bs{x}, \bs{w}) = \bs{c}^\top \bs{H}$ with $\bs{H} = \widehat{\bs{F}}_t(\bs{x}, \bs{w}, \bs{H})$ (i.e., $\widehat{g}_t \sim p( g | \mathcal{D}_{t,1}, \ldots, \mathcal{D}_{t,K} )$).

At each iteration $t$, BONSAI draws two independent realizations of the function network: one for optimizing over design variables and one for optimizing over uncertainty, which are denoted as
\begin{align*}
    \widehat{g}_t^{(x)}(\bs{x}, \bs{w}) &= \bs{c}^\top \bs{H} ~~ \text{where} ~~ \bs{H} = \widehat{\bs{F}}_t^{(x)}(\bs{x}, \bs{w}, \bs{H}), \\
    \widehat{g}_t^{(w)}(\bs{x}, \bs{w}) &= \bs{c}^\top \bs{H} ~~ \text{where} ~~ \bs{H} = \widehat{\bs{F}}_t^{(w)}(\bs{x}, \bs{w}, \bs{H}).
\end{align*}
The first function $\widehat{g}_t^{(x)}$ is used to select the next design point $\bs{x}_{t+1}$ via a max-min optimization. The second function $\widehat{g}_t^{(w)}$ is used to find the corresponding worst-case uncertainty $\bs{w}_{t+1}$ and induced network state $\bs{H}_{t+1}$. Using two independent samples (rather than reusing the same realization)  prevents the algorithm from deterministically ``chasing its tail'' (a risk if the same realization drives both the max and the min) and further encourages exploration of the entire uncertainty set.

\begin{algorithm}[htb!]
\caption{BONSAI}
\textbf{Required:} Sets $\mathcal{X} \subset \mathbb{R}^{n_x}$, $\mathcal{W} \subset \mathbb{R}^{n_w}$, and $\mathcal{H} \subset \mathbb{R}^K$; GP priors $f_k \sim \mathcal{GP}( \mu_{0,k}, \Sigma_{0,k} )$ for all $k = 1,\ldots,K$; Graph representation $(V, E)$ of \eqref{eq:network-functions}; Vector $\bs{c} \in \mathbb{R}^{K}$ in \eqref{eq:network-objective}; Initialization budget $n$; and Total evaluation budget $N$.

\begin{algorithmic}[1]
\State \textbf{Initialize data:} For all $k = 1,\ldots,K$, generate $n$ initial sampling points $\bs{z}_{1:n, k}$ inside of compact set $\mathcal{Z}_k = \mathcal{X}_{I_x(k)} \times \mathcal{W}_{I_w(k)} \times \mathcal{H}_{J(k)}$. Evaluate $f_k$ at these points using \eqref{eq:network-functions} to construct initial dataset $\mathcal{D}_{n,k} = \{ (\bs{z}_{i,k}, h_{i,k}) \}_{i=1}^n$. 
\State \textbf{Train GPs:} Train GP surrogate for $f_k$ using $\mathcal{D}_{n,k}$ for all $k = 1,\ldots, K$. Let $\mathcal{GP}_{n,k} \leftarrow \mathcal{GP}( \mu_{n,k}, \Sigma_{n,k} )$ denote the GP posterior for $f_k$.
\For{$t=n, n+1, \ldots, N-1$}
\For{$k=1,\ldots,K$}
\State Generate two Thompson samples $\widehat{f}^{(x)}_{t,k}, \widehat{f}^{(w)}_{t,k} \sim \mathcal{GP}_{t,k}$.
\EndFor
\State Construct $\widehat{\bs{F}}^{(x)}_t$ and $\widehat{\bs{F}}^{(w)}_t$ using $\widehat{f}^{(x)}_{t,k}$ and $\widehat{f}^{(w)}_{t,k}$ to replace true $f_k$ in \eqref{eq:compact-network}.
\State Solve the following max-min optimization problem:
\begin{align} \label{eq:max-min-acq}
    \bs{x}_{t+1} = \argmax_{\bs{x} \in \mathcal{X}} \min_{(\bs{w}, \bs{H}) \in \mathcal{W} \times \mathbb{R}^K} \bs{c}^\top \bs{H}  ~~ \text{s.t.} ~~ \bs{H} = \widehat{\bs{F}}^{(x)}_t(\bs{x}, \bs{w}, \bs{H}).
\end{align}
\State Solve the following minimization problem:
\begin{align} \label{eq:min-acq}
    ( \bs{w}_{t+1}, \bs{H}_{t+1} ) = \argmin_{(\bs{w}, \bs{H}) \in \mathcal{W} \times \mathbb{R}^K} \bs{c}^\top \bs{H}  ~~ \text{s.t.} ~~ \bs{H} = \widehat{\bs{F}}^{(w)}_t(\bs{x}_{t+1}, \bs{w}, \bs{H}).
\end{align}
\For{$k=1,\ldots,K$}
\State Construct $\bs{z}_{t+1, k}$ from node $k$ subsets of $( \bs{x}_{t+1}, \bs{w}_{t+1}, \bs{H}_{t+1} )$.
\State Evaluate $f_k$ at $\bs{z}_{t+1, k}$ to get $y_{t+1,k}$ (noisy observation).
\State Update data: $\mathcal{D}_{t+1,k} \leftarrow \mathcal{D}_{t,k} \cup \{ (\bs{z}_{t+1, k}, y_{t+1,k}) \}$. 
\State Update GP posterior: $\mathcal{GP}_{t+1,k} \leftarrow \mathcal{GP}( \mu_{t+1,k}, \Sigma_{t+1,k} )$. See \eqref{eq:posterior-gp}.
\EndFor
\EndFor
\State \textbf{Return:} The recommended robust design (approximately solving \eqref{eq:max-min}):
\begin{align} \label{eq:recommender}
    \widetilde{\bs{x}}_N^\star = \argmax_{\bs{x} \in \mathcal{X}} \min_{(\bs{w}, \bs{H}) \in \mathcal{W} \times \mathbb{R}^K} \bs{c}^\top \bs{H}  ~~ \text{s.t.} ~~ \bs{H} = \bs{F}^\mu_N(\bs{x}, \bs{w}, \bs{H}),
\end{align}
where $\bs{F}^\mu_N$ is defined by replacing true $f_k$ with $\mu_{N,k}$ in \eqref{eq:compact-network}.
\end{algorithmic}
\label{alg:bonsai}
\end{algorithm}

This full process is described in Algorithm~\ref{alg:bonsai}. Lines 1 and 2 initialize the GP models from a set of initial evaluations. The main loop (starting at Line 3) iteratively draws the TS realizations (Lines 4 to 7), solves the optimization problems defined in terms of $\widehat{g}_t^{(x)}$ and $\widehat{g}_t^{(w)}$ (Lines 8 and 9), and updates the node-level GPs using the newly collected data (Lines 10 to 16). This procedure continues until the evaluation budget $N$ is exhausted.

After the final iteration, BONSAI returns a recommended design $\widetilde{\bs{x}}_N^\star$ that approximates the solution to \eqref{eq:max-min} (Line 17). In practice, this recommendation must be made using only the surrogate models, since the true objective remains inaccessible outside of data collected through function network evaluations. One reasonable strategy is to define a surrogate objective function by replacing each node function $f_k$ with its posterior mean $\mu_{N,k}$, and then solve the deterministic problem in \eqref{eq:recommender}, i.e., $\bs{c}^\top \bs{H}$ with $\bs{H} = \bs{F}_N^\mu(\bs{x}, \bs{w}, \bs{H})$ approximately equals $\mathbb{E}\{ g(\bs{x},\bs{w}) | \mathcal{D}_{N,1}, \ldots, \mathcal{D}_{N,K} \}$ where $\bs{F}_N^\mu$ denotes the function defining the compact network using node-wise posterior means instead of the true node functions. While this does not yield the exact posterior expectation of $g$, it provides a robust and computationally tractable approximation. This surrogate-based recommendation corresponds to a \textit{risk-neutral} decision rule.

In many settings, however, risk preferences may differ. A risk-averse decision maker may prefer to hedge against worst-case outcomes by evaluating a lower quantile of the objective, while a risk-seeking user may prefer to target more optimistic upper bounds. This can be handled in BONSAI by replacing the recommender with a $\tau$-quantile function for $g$. Lower $\tau$ values correspond to more pessimistic assumptions about the function’s value, e.g., $\tau$ values of 0.05, 0.5, and 0.95 would be considered risk-averse, risk-neutral, and risk-seeking, respectively. 
A differentiable quantile-based composite GP model was recently developed in \cite{Lu2023}, which can be extended to the function network-based setting considered here.

We provide a detailed illustration of BONSAI on the two-dimensional white-box test case in Figure~\ref{fig:grey_box} in \ref{app:bonsai}. 

\begin{remark}
    Although BONSAI is presented for solving max-min problems of the form \eqref{eq:max-min}, the framework can in principle be extended to handle feasibility constraints such as $\max_{\bs{w} \in \mathcal{W}} g_i(\bs{x}, \bs{w}) \leq 0$ where $g_i$ are expressed as function networks (similarly to the objective $g$) using the exact penalty function approach presented in our CARBO method \cite{Kudva2022}. In this case, worst-case constraints would be incorporated through a penalized TS surrogate in \eqref{eq:max-min-acq} and Line 9 of Algorithm \ref{alg:bonsai} would be modified to find the worst-case uncertainty for each unknown function. It is important to note that CARBO relies on rigorous upper/lower confidence bounds, which enable it to achieve strong convergence guarantees that would not carry over directly to BONSAI. The development of such rigorous yet tractable bounds for the function network setting remains an interesting and important direction for future work.
\end{remark}

\begin{remark}
    This work focuses on pure worst-case (adversarial) robustness, where the objective is to perform well even under the most unfavorable uncertainty values. Extensions to the stochastic ($\max_{\bs{x}} \mathbb{E}_{\bs{w} \sim P} \{ g(\bs{x}, \bs{w}) \}$) \cite{kirschner2019stochastic, toscano2022bayesian, cakmak2020bayesian} and distributionally robust ($\max_{\bs{x}} \inf_{ \mathcal{Q} \sim \mathcal{U} } \mathbb{E}_{\bs{w} \sim \mathcal{Q}}\{ g(\bs{x}, \bs{w}) \}$) \cite{kirschner2020distributionally, husain2023distributionally} formulations are conceptually straightforward by simply replacing the black-box assumption on $g$ with a function network \eqref{eq:network-functions}-\eqref{eq:network-objective}; however, they are practically non-trivial. The key difficulty is that this structure induces a non-Gaussian posterior, which breaks many acquisition strategies that rely on closed-form GP-based uncertainty propagation. Moreover, optimizing over expectations or worst-case distributions further increases computational cost. These settings are important directions for future work, especially in applications where uncertainty shifts during deployment (e.g., operational decision-making).
\end{remark}

\begin{remark}
    The choice of the uncertainty set $\mathcal{W}$ in the original problem formulation \eqref{eq:max-min} plays an important role in the practical performance of BONSAI. Larger sets lead to more conservative designs by enforcing stronger robustness guarantees. In practice, $\mathcal{W}$ should be chosen based on the application's tolerance to uncertainty -- this could be derived from physical intuition, expert knowledge, or data-driven confidence sets. Since BONSAI relies on construction of a surrogate model for $g(\bs{x}, \bs{w})$, it is straightforward to assess sensitivity to $\mathcal{W}$ by running post-hoc analyses (e.g., evaluating how the robust solution shifts under different choices of $\mathcal{W}$). Developing systematic ways to learn or adapt $\mathcal{W}$ from data within this framework and/or user preferences is another promising direction for future work. 
\end{remark}

\subsection{Regret and convergence analysis for the nominal setting}

To provide theoretical insight into BONSAI, we analyze a standard performance measure in the BO literature, the Bayesian cumulative regret (BCR):
\begin{align}
    \text{BCR}_T = \mathbb{E}\!\left\lbrace \sum_{t=1}^T g(\bs{x}^\star) - g(\bs{x}_t) \right\rbrace,
\end{align}
where $\bs{x}^\star = \argmax_{\bs{x}\in\mathcal{X}} g(\bs{x})$ is the global maximizer and $\bs{x}_t$ is the point selected at iteration $t$. 
We focus here on the \textit{nominal case}, where there is no uncertainty set (i.e., $\mathcal{W}$ is a singleton), so the inner minimization in \eqref{eq:max-min} is avoided. This isolates the function network aspect of BONSAI and already represents a non-trivial extension beyond existing analyses. Our main result shows that in this setting, BONSAI achieves a finite-time regret bound that cleanly separates into (i) a constant that depends on the network structure and (ii) an information gain term characterizing statistical complexity.

\textbf{Theorem} (Informal). Consider the nominal version of \eqref{eq:max-min} over a finite design space. Let each node function $f_k$ follow an independent GP prior with associated maximum information gain (MIG) $\gamma_{T,k}$, and define $\gamma^{\mathrm{sum}}_T = \sum_{k=1}^K \gamma_{T,k}$. Suppose all queries are corrupted with i.i.d. Gaussian noise and the network satisfies some mild technical regularity conditions. Then, BONSAI (Algorithm~\ref{alg:bonsai}) achieves
\begin{align}\label{eq:bcr-bound}
    \text{BCR}_T = O\!\left(\sqrt{T \gamma^{\mathrm{sum}}_T \log|\mathcal{X}|}\right).
\end{align}
% \textbf{Theorem} (Informal, BCR bound for nominal case). Consider the nominal version of \eqref{eq:max-min} over a finite design space. Let node functions $f_k$ be independent GPs with associated maximum information gain (MIG) $\gamma_{T,k}$ and summed version $\gamma^\text{sum}_{T} = \textstyle\sum_{k=1}^K \gamma_{T,k}$. Also, assume all network output queries are subject to i.i.d. Gaussian noise and the network satisfies certain (relatively mild) regularity assumptions. Then, running the nominal version of BONSAI (Algorithm \ref{alg:bonsai}) results in the following BCR bound:
% \begin{align} \label{eq:bcr-bound}
%     \text{BCR}_T = O\left( \sqrt{T} \gamma^\text{sum}_{T} \log|\mathcal{X}| \right).
% \end{align}

The detailed setup, assumptions, lemma/theorem statements, and proofs are given in \ref{app:theoretical-conv}. Importantly, \eqref{eq:bcr-bound} reduces to the classical TS regret bound in the black-box case ($K=1$) \cite[Theorem~3.1]{takeno2023posterior}, and thus properly generalizes existing results to function networks.

This result is useful because it directly implies \textit{convergence in probability to the global optimum} as long as the MIG decays quickly enough. Specifically, Lemma \ref{lem:BSRvsBCR} in \ref{app:theoretical-conv} shows that if $\text{BCR}_T$ grows sublinearly with $T$, then the Bayesian simple regret (i.e., expected difference between global solution and best-found solution) vanishes. From~\eqref{eq:bcr-bound}, this holds whenever $\gamma^{\mathrm{sum}}_T$ is sublinear in $T$, which is known for many popular kernels such as the squared exponential and Matern-$\nu$ families~\cite{srinivas2010gaussian}. To our knowledge, this is the first finite-time regret bound established for optimization over function networks, and it applies to a broad class of structures, including networks with cycles (which have not been previously analyzed).

Finally, we comment on why this analysis does not yet extend to the robust case. When $\mathcal{W}$ is non-trivial, regret is no longer directly measurable because of the inner minimization, and one must instead define and bound \textit{robust regret}. Existing approaches (e.g., \cite{Stableopt}, \cite{Kudva2022}) rely on confidence bounds, which can be analytically derived in the single-node case but are much harder to obtain for function networks since $g(\bs{x},\bs{w})$ is no longer Gaussian. Optimization-based approaches to construct such bounds, e.g., \cite{MCBO} may provide a practical route, but a full theoretical treatment is beyond the scope of this work. Nonetheless, the nominal regret bound motivates why TS can perform well in practice even in the robust two-stage setting, which is a conclusion we also empirically observe in our numerical experiments (Section \ref{sec:numerical-experiments}). 

\subsection{Implementation details}

The BONSAI method (Algorithm \ref{alg:bonsai}) involves a number of internal steps that we discuss in more detail in this section.

\paragraph{GP model selection and training} The choice of GP priors plays a central role in the performance of BONSAI, as each surrogate model must accurately represent its corresponding node function given limited observations. Unless otherwise specified, we adopt a zero-mean prior, $\mu_0(\bs{z}) = 0$, which is a standard assumption that does not limit generality when the training data are normalized (typically using z-score scaling) \cite{Bradford2018}. This prior simplifies inference while centering the surrogate on the data-driven signal.

For the covariance function, we use the Mat\'ern 5/2 kernel as the default choice due to its favorable balance between smoothness and flexibility, i.e., 
\begin{align*}
    \Sigma_{\text{Mat\'ern},5/2}(\bs{z}, \bs{z}') = \zeta^2 \left( 1 + \sqrt{5} r_{\bs{z},\bs{z}'} + \frac{5}{3} r_{\bs{z},\bs{z}'}^2 \right) \exp\left( -\sqrt{5} r_{\bs{z},\bs{z}'} \right),
\end{align*}
where $\zeta^2$ is the output variance and $r_{\bs{z},\bs{z}'} = \sqrt{ (\bs{z} - \bs{z}')^\top L^{-2} (\bs{z} - \bs{z}') }$ is the scaled Euclidean distance between inputs $\bs{z}$ and $\bs{z}'$, with $L = \text{diag}(\ell_1, \ldots, \ell_{n_z})$ being a diagonal matrix of lengthscale hyperparameters. Each lengthscale $\ell_i$ controls the sensitivity of the GP to variation in the $i$th input dimension, while the scale parameter $\zeta$ governs the overall magnitude of fluctuations in the function.
The covariance kernel encodes key assumptions about the smoothness and structure of the target function (often referred to as inductive biases) and should be selected carefully based on prior knowledge or validated through model performance. More expressive (e.g., non-stationary) covariance functions, such as those induced by the deep kernel learning framework \cite{wilson2016deep}, can be beneficial in high-dimensional or highly nonlinear settings.

All kernel hyperparameters (e.g., $\zeta$, $\ell_1,\ldots,\ell_{n_z}$) are learned individually for each node function $f_k$ by maximizing the marginal log-likelihood (MLL) of the observed data (see, e.g., \cite{Rasmussen2006} for details). While BONSAI supports heterogeneous GP surrogates -- allowing different kernels and hyperparameters per node -- the default configuration provides a robust and scalable baseline that works well across a broad range of applications.

\paragraph{Efficient Thompson sampling} A core component of BONSAI is solving the acquisition subproblems \eqref{eq:max-min-acq} and \eqref{eq:min-acq}, both of which require drawing realizations from the posterior distribution of each node function. To enable scalable and differentiable sampling in this setting, we adopt the method proposed by \cite{wilson2020efficiently}, which generates approximate posterior samples using a combination of random Fourier features (RFF) and exact GP updates.

The method proceeds in two stages. First, the GP prior for any function $f$, i.e.,  $f(\bs{z}) \sim \mathcal{GP}(0, \Sigma(\bs{z}, \bs{z}'))$ is expressed in weight-space form using Bochner’s theorem, which states that any stationary kernel $\kappa(\bs{z} - \bs{z}')$ can be represented as the expected dot product of sinusoidal features. Specifically, the kernel is approximated by a set of $D$ features as \cite{rahimi2007random}
\begin{align*}
    \kappa(\bs{z} - \bs{z}') &\approx \bs{\phi}(\bs{z})^\top \bs{\phi}(\bs{z}'), \\
    \bs{\phi}(\bs{z}) &= \sqrt{\frac{2}{D}} \left[ \cos(\bs{\omega}_1^\top \bs{z} + b_1), \ldots, \cos(\bs{\omega}_D^\top \bs{z} + b_D) \right]^\top,
\end{align*}
where each $\bs{\omega}_d \in \mathbb{R}^{n_z}$ is drawn from the spectral density of the kernel and $b_d \sim \text{Uniform}[0, 2\pi]$. Sampling a weight vector $\bs{\theta} \sim \mathcal{N}(0, I_D)$ then defines a prior function draw as $f_{\text{prior}}(\bs{z}) = \bs{\phi}(\bs{z})^\top \bs{\theta}$.

The second stage performs a function-space correction to ensure that the sample is consistent with any observed training data $\{ \bs{Z}, \bs{y} \} = \{ \bs{z}_i, y_i \}_{i=1}^{N}$. Let $\bs{\Phi} = [ \bs{\phi}(\bs{z}_1), \ldots, \bs{\phi}(\bs{z}_{N}) ]^\top \in \mathbb{R}^{N \times D}$ be the feature matrix. The posterior function sample $\widehat{f} \sim f | \bs{Z}, \bs{y}$ is then given by
\begin{align*}
    \widehat{f}(\bs{z}) = \bs{\phi}(\bs{z})^\top \bs{\theta} 
    + \Sigma(\bs{z}, \bs{Z})^\top \Sigma( \bs{Z}, \bs{Z} )^{-1} ( \bs{y} - \bs{\Phi} \bs{\theta} ).
\end{align*}
This expression consists of a global sample from the approximate prior plus a corrective term that exactly interpolates the training data using the full GP covariance function. The result is a fully differentiable posterior sample $\widehat{f}(\bs{z})$, which can be used within the BONSAI acquisition problems. Because the procedure is linear in the number of features and highly parallelizable across GPs, it scales well to multiple, repeated samples.

\paragraph{Optimizing the acquisition functions} Each BONSAI iteration requires solving two acquisition subproblems: a two-stage robust design problem \eqref{eq:max-min-acq} (Line 8) and a single-level uncertainty minimization problem \eqref{eq:min-acq} (Line 9). The second subproblem (i.e., minimizing over $\bs{w}$ for a fixed design $\bs{x}$) is relatively straightforward and can be solved using standard multi-start gradient-based optimization methods. We follow the default strategy from BoTorch \cite{balandat2020botorch}, which combines quasi-random Sobol sequence-based initialization with L-BFGS-B \cite{zhu1997algorithm} as the local solver. Since this problem is typically low-dimensional and smooth, it is well-suited to such solvers.

A practical challenge arises when the function network includes equality constraints, i.e., when the node outputs $\bs{H}$ are defined implicitly for cyclic graphs. In the \textit{acyclic} case, we can avoid explicitly solving the constraints by topologically sorting the graph and sequentially evaluating each node, using previously computed values for parent nodes. In contrast, the \textit{cyclic} case requires solving the system $\bs{H} = \widehat{\bs{F}}_{t}(\bs{x}, \bs{w}, \bs{H})$ ($\widehat{\bs{F}}_{t}$ is a generic TS for the function network, which will be different for the two acquisition problems) using a fixed-point solver or root-finding routine that must be embedded inside the acquisition optimization. In such cases, we need to replace L-BFGS-B with a general-purpose nonlinear programming solver such as IPOPT \cite{wachter2006implementation}, which can directly handle nonlinear equality constraints.

The robust design step (i.e., maximizing over $\bs{x} \in \mathcal{X}$) is much more computationally demanding due to its two-stage structure. Although this problem can be formulated as a generalized semi-infinite program (GSIP) \cite{vazquez2008generalized}, such formulations are impractical in the robust BO setting, as they require globally solving the inner problem at every outer iteration. Instead, we approximate the inner minimization over uncertainty by discretizing the uncertainty set $\mathcal{W}$ into a finite subset $\mathcal{W}_{\text{disc}} = \{ \bs{w}_1, \ldots, \bs{w}_m \}$, which could be predefined, sampled randomly, or constructed using deterministic quadrature. This results in a surrogate acquisition function
\begin{align} \label{eq:discrete-max-approx}
    \widehat{G}_t^{(x)}(\bs{x}) = \min_{j = 1,\ldots,m} \widehat{g}_t^{(x)}(\bs{x}, \bs{w}_j).
\end{align}
Thanks to efficient tensorized implementations in PyTorch \cite{NEURIPS2019_bdbca288}, this discrete form can be evaluated for several thousands or more $\bs{w}_j$ values with minimal overhead. However, the hard minimum in \eqref{eq:discrete-max-approx} is non-differentiable, which hinders the use of gradient-based optimization.

In the acyclic case, we can exploit the fact that $\widehat{g}_t^{(x)}(\bs{x}, \bs{w}_j)$ is differentiable with respect to $\bs{x}$ by construction since the computational graph is fully unrolled during the forward propagation step. We can then approximate \eqref{eq:discrete-max-approx} using the smooth ``fat minimum'' operator proposed in \cite{AmentUnexpected2023}. Let $\bs{q} = [q_1, \ldots, q_m]^\top \in \mathbb{R}^m$ be a vector of $m$ values. The fat maximum is defined as
\begin{align}
    \varphi_{\max}(\bs{q}) = \max_{j} q_j + \tau \log \left( \sum_{i=1}^m \left[ 1 + \left( \frac{q_i - \max_j q_j}{\tau} \right)^2 \right]^{-1} \right),
\end{align}
and the corresponding fat minimum is given by \( \varphi_{\min}(\bs{q}) = -\varphi_{\max}(-\bs{q}) \). This allows us to define a smooth surrogate
\begin{align}
    \widetilde{G}_t^{(x)}(\bs{x}) = \varphi_{\min}\left( \left[ \widehat{g}_t^{(x)}(\bs{x}, \bs{w}_1), \ldots, \widehat{g}_t^{(x)}(\bs{x}, \bs{w}_m) \right]^\top \right),
\end{align}
which can be optimized using multi-start L-BFGS, as the gradients of the objective with respect to \( \bs{x} \) are well-defined and efficient to compute.

In the cyclic case, however, the node outputs $\bs{H}$ are defined implicitly through fixed-point equations. As a result, even though $\widehat{g}_t^{(x)}(\bs{x}, \bs{w})$ can be evaluated for any $\bs{w}$, it is no longer differentiable with respect to $\bs{x}$ using standard automatic differentiation tools. In this setting, the fat minimum approximation is no longer usable with gradient-based methods. Instead, we suggest utilizing derivative-free optimization in this case where the outer maximization is solved using a global black-box solver such as a genetic algorithm \cite{holland1992genetic} or CMA-ES \cite{hansen2001completely}. This approach is slower but avoids the need to backpropagate through the fixed-point solver. We plan to explore more efficient robust optimization approaches for the cyclic case in future work.

Finally, we note that in the acyclic case, Line 11 of Algorithm~\ref{alg:bonsai} (which constructs the new input points $\bs{z}_{t+1,k}$ at which the potentially expensive node functions should be evaluated) can also be simplified. Since node outputs are evaluated sequentially, we do not need to build the full vector $\bs{H}_{t+1}$ in advance. Instead, we can recursively compute each node’s output and record its input-output pair as the topological sort is traversed. This improves both memory and implementation efficiency in the case of acyclic graphs.

% \subsection{Extension to constrained robust optimization}

\section{Numerical Experiments}
\label{sec:numerical-experiments}

\subsection{Baseline methods and performance metrics}

We evaluate the performance of BONSAI against three robust optimization baselines and two nominal BO strategies. These baselines are described below:
\begin{itemize}
    \item \textbf{Random}. The next sample point \(( \bs{x}_{t+1}, \bs{w}_{t+1} )\) is drawn uniformly at random from the product space \(\mathcal{X} \times \mathcal{W}\). The final recommended design follows the same strategy used by BONSAI in \eqref{eq:recommender}. This baseline highlights the performance losses incurred when data is passively collected rather than actively chosen to solve the optimization task.
    \item \textbf{ARBO-GP-Quantile}. This is a fully black-box robust BO method based on alternating confidence bounds, as proposed in \cite[Algorithm 1]{Paulson2022}. It uses the same two-stage acquisition optimization structure as BONSAI but applies it to a GP model defined directly over $g(\bs{x}, \bs{w})$. Thompson sampling is replaced with an alternating upper/lower confidence bound strategy, where we fix the standard deviation multiplier at 2 to correspond to an approximate 95\% confidence level. A Mat\'ern 5/2 covariance function is used throughout.
    \item \textbf{ARBO-GP}. This variant uses the same acquisition procedure as ARBO-GP-Quantile, but replaces the robust recommendation rule with a mean-based rule akin to that used in BONSAI (except applied to the black-box GP model for $g$).
    \item \textbf{LogEI}. A standard nominal BO method based on the log-Expected Improvement acquisition function introduced in \cite{AmentUnexpected2023}. To apply logEI to the robust optimization setting considered in this work, we fix the uncertain variables at a nominal value $\bar{\bs{w}}$, resulting in the approximate problem $\max_{\bs{x} \in \mathcal{X}} g(\bs{x}, \bar{\bs{w}})$. The recommended design is selected as the best observed $\bs{x}$ with respect to this nominal objective.
    \item \textbf{EIFN}. A nominal BO method that leverages the EIFN acquisition function proposed in \cite{BOFN}, which builds a GP over a known function network. As with logEI, the uncertain variables are fixed to a nominal value and the recommended design is the best sampled \(\bs{x}\) based on the nominal objective.
\end{itemize}
The two nominal baselines (logEI and EIFN) are included to highlight the importance of incorporating uncertainty into the optimization process. We focus on these particular surrogate-based methods because many alternative approaches to robust black-box optimization (such as local or evolutionary search strategies) are not designed for sample efficiency, particularly in the low-dimensional, limited-budget settings considered here (see \ref{app:surrogate-free-comparison} for a further discussion and comparison).

All methods are initialized using $n = 2n_x + 2n_w + 1$ samples, drawn uniformly at random over the corresponding input space. The total evaluation budget is fixed at 100 evaluations for each method on each problem instance.
We assess performance using the worst-case objective for the recommended design, i.e., $\max_{\bs{w} \in \mathcal{W}} g(\widetilde{\bs{x}}_N^\star, \bs{w})$ where $\widetilde{\bs{x}}_N^\star$ is the final recommended design from the given method (e.g., BONSAI uses \eqref{eq:recommender}). 
% When $\bs{x}^\star$ is unknown (as in real-world case studies), we substitute the best-found solution across all replicates and methods as a proxy.
We evaluate this metric every 5 iterations to monitor convergence and report its trajectory throughout the optimization process. To assess the variability introduced by random initialization and sampling, each method is repeated over 30 independent trials using different random seeds. We report the mean worst-case recommended objective value along with 95\% confidence intervals computed using the standard error formula.

All BO related methods were implemented using the BoTorch \cite{balandat2020botorch} package. The code to reproduce all of our numerical experiments is available at \href{https://github.com/PaulsonLab/BONSAI}{https://github.com/PaulsonLab/BONSAI}.

\subsection{Case studies}

\subsubsection{Synthetic test functions}

We construct four synthetic test problems by adapting well-known benchmark functions from the global and robust optimization literature into function networks. These networks vary in dimensionality and complexity, offering a range of challenges to evaluate BONSAI’s performance. A visual summary of each function network is provided in Figure~\ref{fig:syntheic_function_network}, and full mathematical details are included in \ref{app:synthetic-test}.

\textbf{Polynomial.} This problem consists of $K = 4$ nodes, with $n_x = 2$ design variables and $n_w = 2$ uncertainties. The first three nodes are black-box components, and the final node is a white-box aggregation.

\textbf{Cliff.} This higher-dimensional network includes $K = 6$ nodes with $n_x = 5$ design variables and $n_w = 5$ uncertainties. The first five nodes are black-box components, followed by a white-box final node.

\textbf{Rosenbrock.} A compact network with $K = 4$ nodes, this problem uses $n_x = 1$ design variable and $n_w = 2$ uncertainties. As before, the last node is a white-box aggregator applied to three upstream black-box nodes.

\textbf{Modified Sine.} This example features $K = 7$ nodes, $n_x = 2$ design variables, and $n_w = 2$ uncertainties. The first six nodes are black-box functions, while the final node performs white-box aggregation.

\begin{figure}[htb!]
  \centering
  \includegraphics[width=1\textwidth]{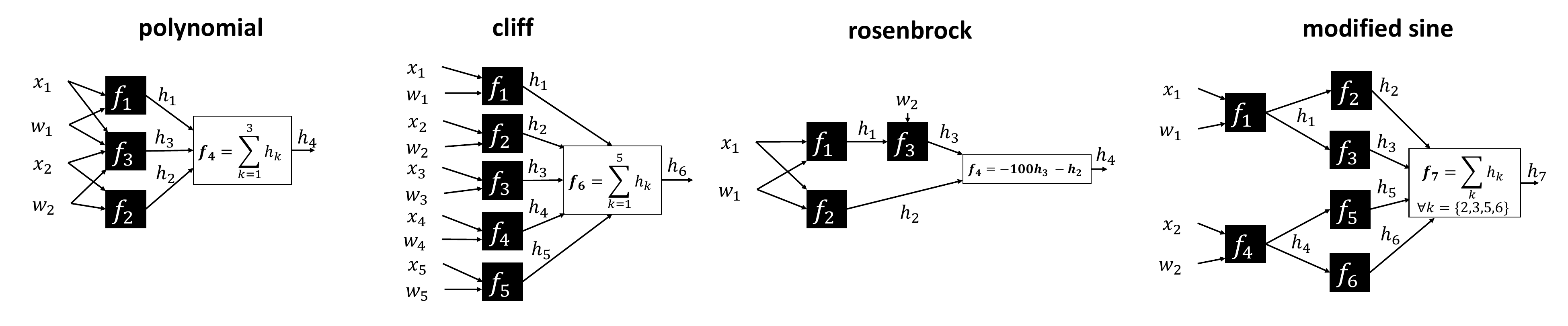}
  \caption{Function network structures for the synthetic test problems. Nodes correspond to component functions and directed edges indicate dependency. White-box aggregators are shown in white.}
  \label{fig:syntheic_function_network}
\end{figure}

To further illustrate the modeling challenges posed by these synthetic benchmarks, Figure~\ref{fig:two_D_true} shows contour plots of the projected worst-case objective $G(\bs{x}) = \min_{\bs{w} \in \mathcal{W}} g(\bs{x}, \bs{w})$ for the Polynomial and Modified Sine problems. These two cases were selected because they offer 2-dimensional design domains that are easily visualized. Both landscapes exhibit strong nonstationarity, with sharp variations in some regions and flat behavior in others, leading to objective functions that are highly nonlinear, anisotropic, and difficult to model using global GP surrogates. As seen in the figure, the robust optimum $\bs{x}^\star$ (marked by a star) lies near regions of high variability, further highlighting the importance of an adaptive, structure-aware optimization strategy.

\begin{figure}[htb!]
  \centering
  \includegraphics[width=1\textwidth]{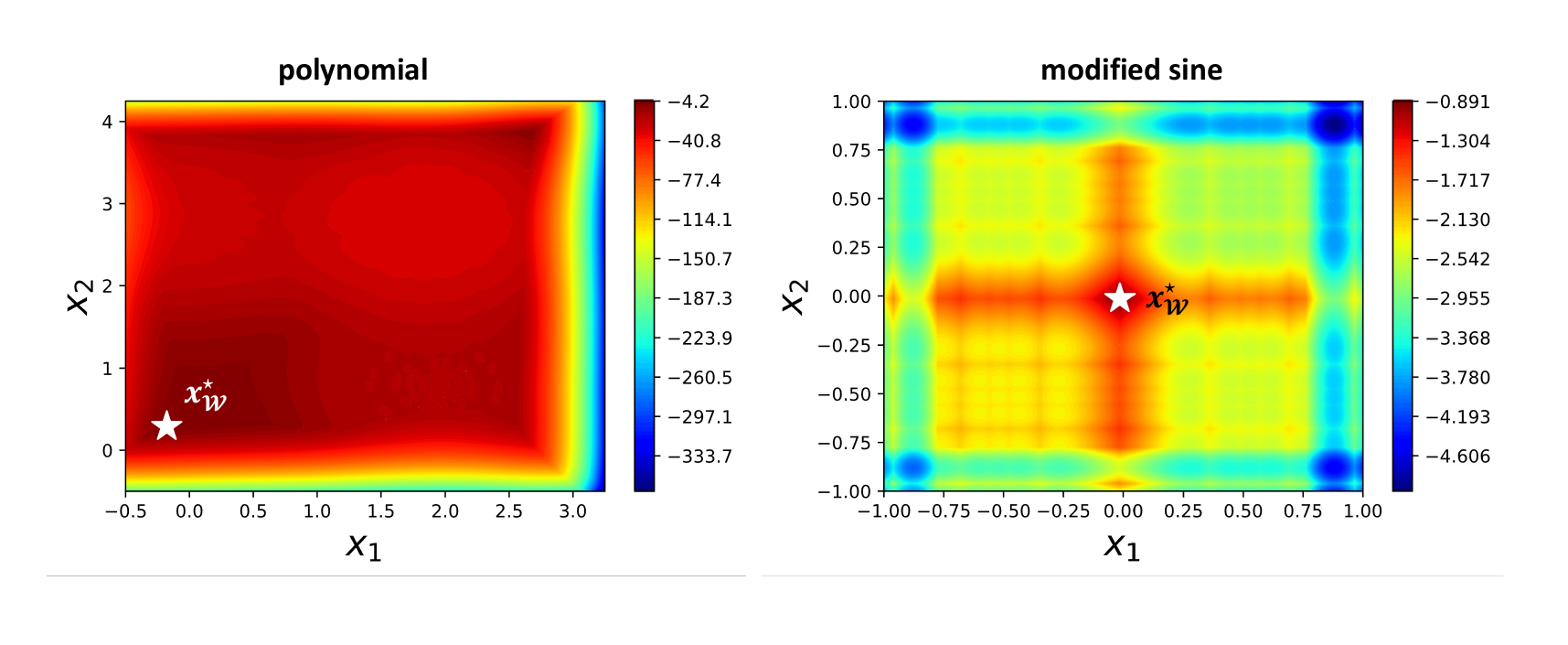}
  \caption{Contour plots of the robust objective $G(\bs{x}) = \min_{\bs{w} \in \mathcal{W}} g(\bs{x}, \bs{w})$ for the Polynomial (left) and Modified Sine (right) problems. The robust optimum $\bs{x}^\star$ is marked by a star. These plots reveal the strong nonstationarity and multimodality that make these functions particularly difficult for existing black-box robust optimization methods.}
  \label{fig:two_D_true}
\end{figure}

\subsubsection{Robot pushing under uncertainty}

We consider a robust control problem adapted from the robot pushing benchmark in \cite{Zi2017, LIS3}, where the goal is to determine a control policy that reliably pushes an object to an unknown target location on a 2-dimensional plane. The object begins at a fixed location $(-0.5, -2)$, while the robot's initial position $(x_1, x_2)$ and pushing duration $x_3$ are treated as design variables. The design domain is given by $\mathcal{X} = [-5, 5]^2 \times [1, 70]$.

The system is modeled as a function network with six black-box nodes and a white-box final node. The black-box nodes simulate the robot and object dynamics based on discrete-time integration at three time points: $x_3/3$, $2x_3/3$, and $x_3$. The output of this simulation is the final object position $r_f \in \mathbb{R}^2$. The final node computes the negative distance to an unknown target location $\bs{w}$ that represents our uncertainty in this problem. As such, our objective function can be expressed as $g(\bs{x}, \bs{w}) = -\left\| r_f(\bs{x}) - \bs{w} \right\|_2$.
The goal is to maximize the worst-case objective over $\mathcal{W}$, yielding the equivalent max-min formulation in \eqref{eq:max-min}. The uncertainty set $\mathcal{W}$ consists of 320 known target locations that represent potential object goals. These are generated from a mixture of two components: (i) 20 samples drawn from a bivariate normal distribution and (ii) 300 samples drawn uniformly from a rectangular region. This construction reflects both concentrated and dispersed uncertainties in the robot's operating environment.
To benchmark against nominal methods (e.g., logEI and EIFN), we define a representative nominal target as the centroid of the uncertainty set $\bar{\bs{w}} = \frac{1}{|\mathcal{W}|} \sum_{\bs{w} \in \mathcal{W}} \bs{w}$. 

An illustration of the function network structure for this problem is shown on the right side of Figure~\ref{fig:heatX_robot}. 

\begin{figure}[htb!]
  \centering
  \includegraphics[width=1\textwidth]{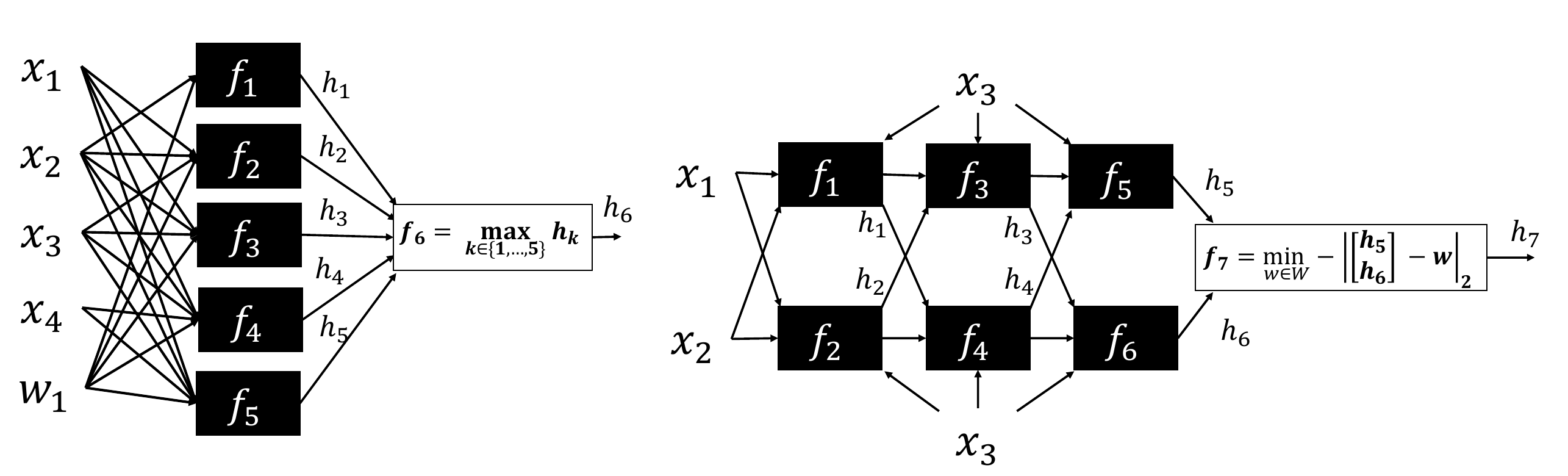}
  \caption{Function network representation of the heat exchanger network flexibility problem (left) and the robust robot pushing problem (right).}
  \label{fig:heatX_robot}
\end{figure}

\subsubsection{Heat exchanger network flexibility}

This case study adapts a classic heat exchanger network (HEN) flexibility test problem originally proposed in \cite{Grossmann1987}. The goal is to determine whether a proposed design can maintain feasible operation across a specified range of uncertainty, a critical concern in process system design. Traditionally, the flexibility test is cast as a tri-level (max–min–max) optimization problem: the outer maximization is over uncertain variables (mapped to $\bs{x}$ in our framework), the middle minimization is over recourse or control variables (mapped to $\bs{w}$), and the innermost maximization is over all inequality constraints that are required to be $\leq 0$ to be considered feasible. If the worst-case constraint violation can be driven to zero or below, the design is considered flexible.

In our BONSAI framework, we recast this tri-level formulation as a two-stage robust optimization problem over a function network. Specifically, we define the objective as $g(\bs{x}, \bs{w}) = \max_{k=1,\ldots,K} f_k(\bs{x}, \bs{w})$ where each $f_k$ corresponds to an inequality constraint function. While these constraints can be derived analytically from energy balances applied to the HEN, we treat them as black-box components in this study to simulate the more general setting of simulation-based constraints. The final objective function $g(\bs{x}, \bs{w})$ - the maximum constraint violation - is treated as a white-box aggregation node that composes the outputs of the five black-box nodes, as illustrated in the formulation of the function network (Figure~\ref{fig:heatX_robot}, left).

The design variables $\bs{x} \in \mathbb{R}^4$ represent inlet temperatures for four process streams: $T_1$, $T_3$, $T_5$, and $T_8$. The recourse variable $\bs{w}$ corresponds to the cooling heat load. All variable bounds and constraint formulations are taken from the original study in \cite{Grossmann1987}. We assume a nominal cooling heat load of 90 kW/K.
The optimal robust design is given by $\bs{x}^\star = [615, 383, 578, 318]^\top$ K. For this design, the worst-case objective is $\min_{\bs{w} \in \mathcal{W}} g(\bs{x}^\star, \bs{w}) = 0$, which indicates that the design is flexible. That is, for all realizations of $\bs{w}$ in the uncertainty set, there exists a feasible recourse action that satisfies all constraints.

% \begin{figure}[htb!]
%   \centering
%   \includegraphics[width=0.75\textwidth]{figures/HeatX.pdf}
%   \caption{Schematic of the heat exchanger network problem adapted from \cite{Grossmann1987}. Each unit corresponds to a specific heat exchange operation between process and utility streams.}
%   \label{fig:HeatX}
% \end{figure}

\subsubsection{Robust neural network hyperparameter tuning}

We consider the problem of robust neural network (NN) hyperparameter tuning under deployment-specific perturbations. This problem is motivated by practical scenarios, such as embedded vision systems or medical imaging diagnostics, where trained models are deployed in environments that may deviate from the clean training setting. Our goal is to identify design variables that yield high test accuracy under the worst-case deployment conditions.

The network architecture is fixed to a feedforward NN with two hidden layers. We optimize over three design variables: learning rate ($x_1$), number of nodes per hidden layer ($x_2$), and the $L_2$ regularization coefficient ($x_3$). The network is trained on clean grayscale images from the UCI Digits dataset \cite{scikit_learn}, where each image is an $8 \times 8$ grid with pixel intensities scaled to $[0, 16]$.

Deployment uncertainty is modeled via image perturbations, represented as two unknown environments: $f_1$ and $f_2$. In each case, additive uniform noise in the range $[0,5]$ is applied to a distinct subset of image rows. For environment $f_1$, the perturbation set $\mathcal{W}_1$ consists of four discrete cases in which noise is applied to adjacent row pairs: rows 1-2, 3-4, 5-6, and 7-8. For environment $f_2$, the perturbation set $\mathcal{W}_2$ defines four different cases where the same magnitude of noise is applied to alternating rows: (1,3), (2,4), (5,7), and (6,8). We express these as a finite uncertainty set
\begin{align*}
    \mathcal{W} = \{ w_1^{(1)}, w_1^{(2)}, w_1^{(3)}, w_1^{(4)}, w_2^{(1)}, w_2^{(2)}, w_2^{(3)}, w_2^{(4)} \},
\end{align*}
where each $w_k^{(i)}$ corresponds to a distinct test-time perturbation configuration. 

See Figure \ref{fig:classifier_network} for an illustration of the complete function network pipeline for this problem.
The function network includes three nodes: two black-box nodes representing the environment-specific classification accuracy functions and one white-box node computing the average accuracy. The objective is to identify the hyperparameters $\bs{x}$ that maximize the worst-case average classification accuracy over $\mathcal{W}$. The nominal uncertainty value $\bar{\bs{w}}$ simply neglects the noise in the two environments. This case study illustrates how BONSAI can be used to tune machine learning models for deployment across heterogeneous conditions.

\begin{figure}[htb!]
  \centering
  \includegraphics[width=1\textwidth]{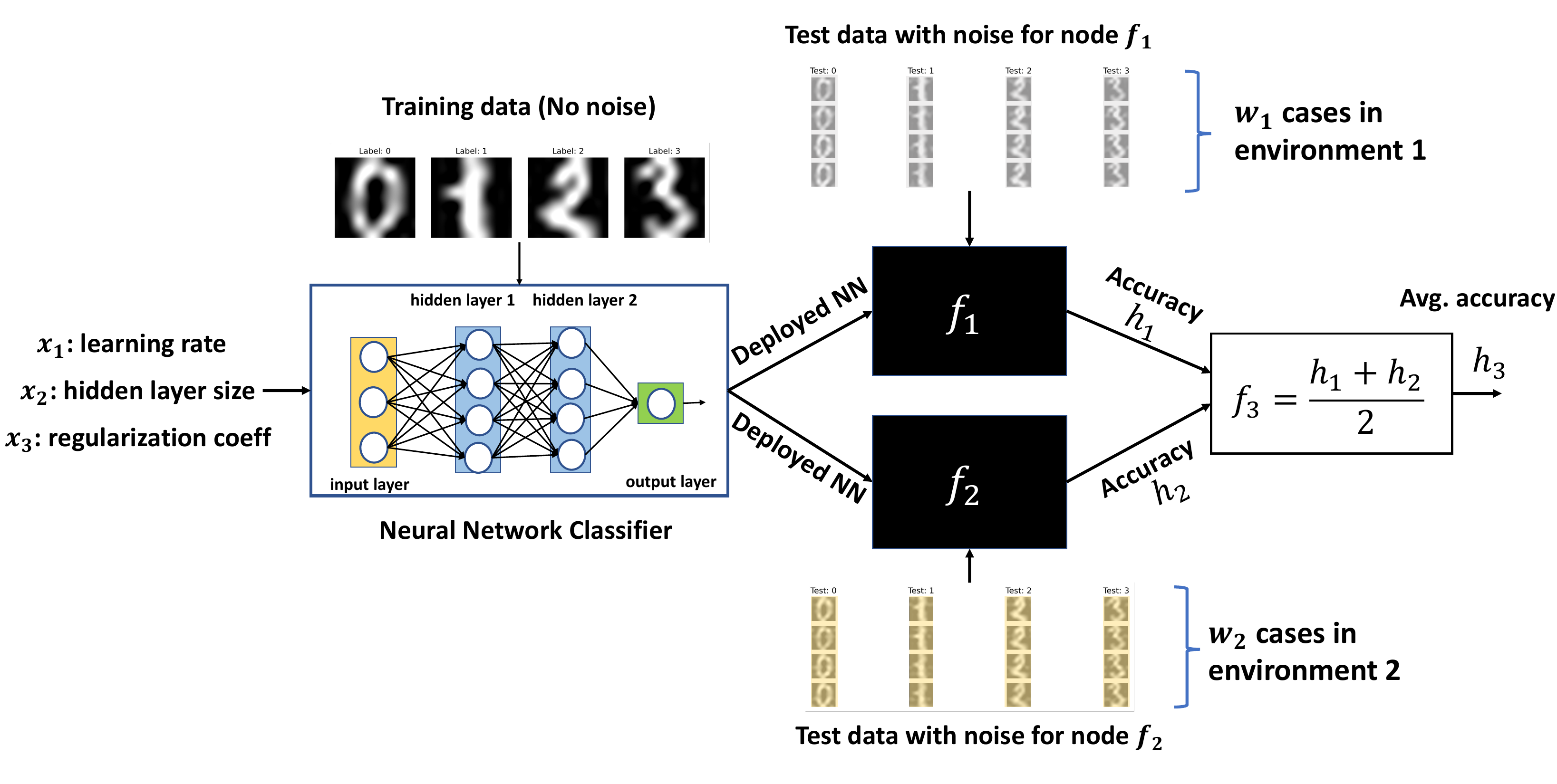}
  \caption{Function network representation of the robust neural network (NN) hyperparameter tuning problem. A NN is trained under specific hyperparameter settings on clean training data and evaluated under deployment-specific structured noise. Accuracy is computed separately in each environment and aggregated to form a scalar performance metric. BONSAI selects hyperparameters to maximize average accuracy under worst-case test-time perturbations across a discrete set of known noise cases.}
  \label{fig:classifier_network}
\end{figure}

\subsubsection{Robust design of a vibration absorber}

Dynamic vibration absorbers are widely used in engineering systems to attenuate resonant oscillations, such as roll motions in ships, structural vibrations in buildings, and torsional fluctuations in rotating machinery \cite{Pennestr1998}. In this case study, adapted from Example 4.2 in \cite{Marzat2015}, we consider a classical mechanical setup: a large primary mass is excited by a harmonic input of uncertain frequency and a secondary mass-spring-damper system is introduced to counteract the resulting vibrations.

The design objective is to tune two absorber parameters -- damping ratio $x_1$ and natural frequency $x_2$, so that the amplitude of oscillations in the primary mass is minimized under worst-case excitation frequency. That is, we seek a design that is robust to variations in the environmental disturbance frequency $w_1$. 
The system is modeled as a function network composed of three black-box nodes and a white-box aggregation node. The black-box nodes compute intermediate dynamical quantities derived from the physical equations of motion
\begin{align*}
    h_1 &= \sqrt{\left(1 - \frac{w_1^2}{x_2^2}\right)^2 + 4 \left( \frac{x_1 w_1}{x_2} \right)^2}, \\
    h_2 &= \frac{w_1^2}{x_2^2}(w_1^2 - 1) - w_1^2(1 + c_1) - 4\frac{x_1 c_2 w_1^2}{x_2} + 1, \\
    h_3 &= \frac{c_2 w_1^3}{x_2^2} + \frac{x_1 w_1^3(1 + c_1) - x_1 w_1}{x_2} - c_2 w_1,
\end{align*}
where $c_1 = c_2 = 0.1$ are fixed system constants. The outputs from these nodes are aggregated through a white-box function that computes the final performance index (negative amplitude of the primary mass)
\begin{align*}
    h_4 = -\frac{1}{\sqrt{h_2^2 + 4h_3^2}} h_1.
\end{align*}

The design space is given by $\mathcal{X} = [0.05, 0.5] \times [1.0, 2.0]$. The uncertainty set $\mathcal{W}$ is discretized over the interval of $[0.05, 2.5]$ as follows
\begin{align*}
    \mathcal{W} = \left\{ w_1 \in [0.05, 2.5] : w_1 = 0.05 + 0.05 k, \; k = 0, 1, \ldots, 49 \right\}.
\end{align*}
This corresponds to 50 evenly spaced frequency values, selected to ensure dense coverage of the uncertainty range while maintaining computational tractability in robust evaluation.

The resulting function network consists of $K = 4$ nodes (three black-box and one white-box), with $n_x = 2$ design variables and $n_w = 1$ uncertain variable. The true robust solution (computed via exhaustive search over a fine grid) is found to be $\bs{x}^\star \approx [0.199, 0.862]^\top$, with a corresponding worst-case objective of approximately $-2.623$. The nominal uncertainty value is 1.275, which is the midpoint of the uncertainty set.

\subsection{Results and discussion}

Figure~\ref{fig:robust-reward} shows the evolution of the worst-case recommended objective value over 100 evaluations for each method across all case studies. Since we are solving a max-min problem, higher values correspond to better performance. BONSAI consistently achieves the highest robust objective values in all domains, indicating its ability to identify designs that remain effective even under adversarial uncertainty.
We observe especially strong performance gains on the two real-world process systems engineering problems -- the HEN and vibration absorber design problems. These domains are characterized by complex physical models with coupled nonlinear constraints. BONSAI’s ability to exploit intermediate structure allows it to navigate these search spaces more efficiently, leading to substantially more robust designs than either nominal or black-box robust baselines.

\begin{figure}[htb!]
  \centering
  \includegraphics[width=1\textwidth]{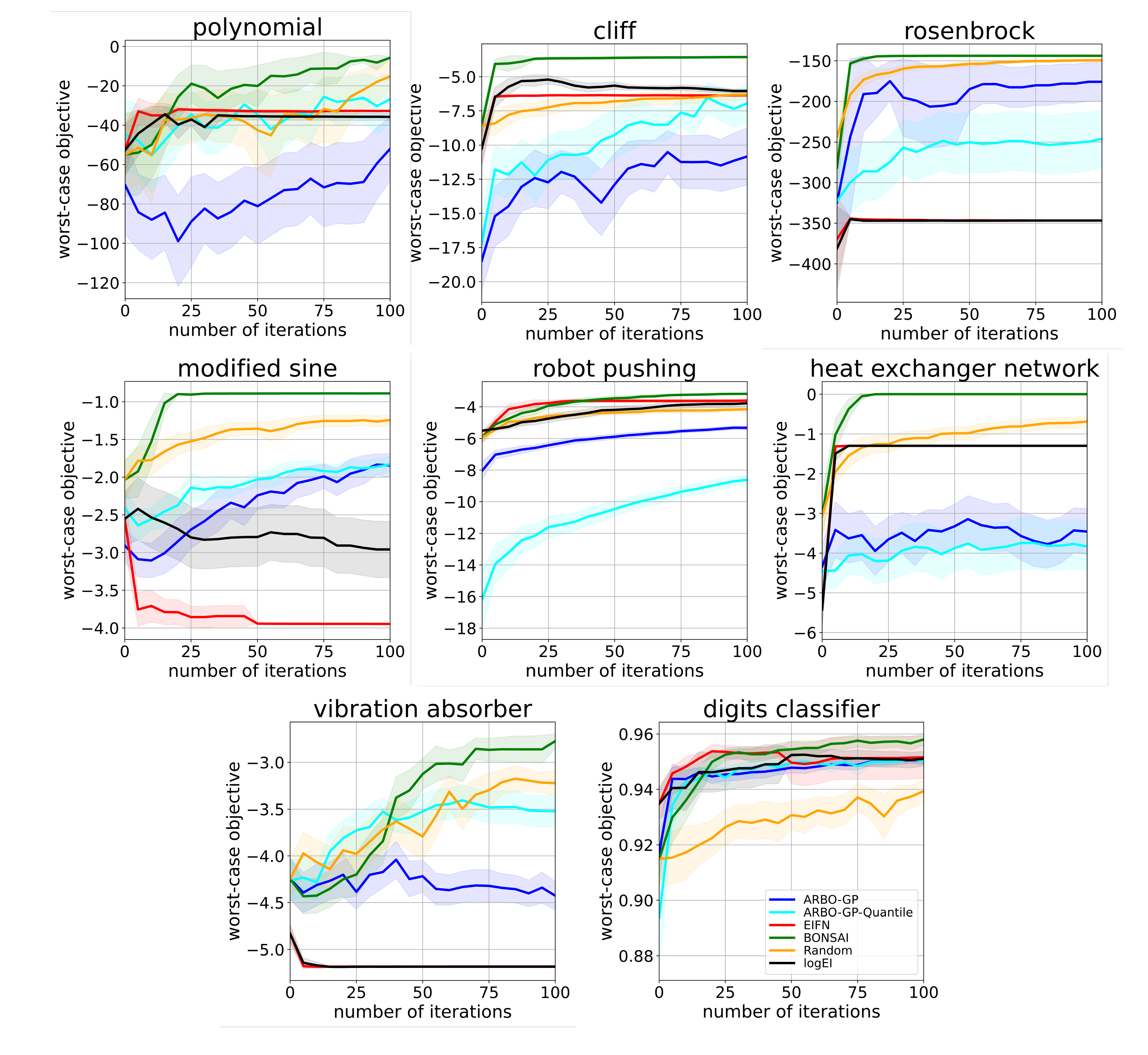}
  \caption{Evolution of the worst-case objective value for the recommended design over 100 evaluations for all algorithms. Curves denote the average performance over 30 random seeds; shaded regions correspond to the estimated 95\% confidence region. Higher values are better. BONSAI consistently outperforms both black-box robust and nominal baselines across all problems. Notably, it shows strong gains in structured engineering problems such as the heat exchanger network and vibration absorber.}
  \label{fig:robust-reward}
\end{figure}

Nominal optimization methods such as logEI and EIFN converge rapidly, but their performance plateaus early and remains far from the true robust optimum. EIFN benefits modestly from structural modeling, but like logEI, fails to account for variability due to the uncertain variables, making the recommended designs perform poorly when subjected to adversarial uncertainty. Robust black-box approaches such as ARBO-GP and ARBO-GP-Quantile capture uncertainty but remain sample-inefficient due to their inability to leverage the internal structural information in the function network. 
Interestingly, combining random sampling with the BONSAI recommendation rule (i.e., using GP function network posteriors without structured acquisition) still yields better results than all non-BONSAI baselines, emphasizing the importance of structure-aware (function network) recommendation. 

In \ref{app:ablation}, we conduct two ablation studies to further investigate the contributions of BONSAI’s acquisition function optimization and recommendation mechanisms. The first study evaluates different combinations of search strategies and recommendation rules, showing that BONSAI’s structure-aware search paired with a function network-based recommendation yields the most robust designs. The second study examines a quantile-based function network recommendation strategy, demonstrating that risk-averse recommendations are particularly valuable in early iterations when model uncertainty is high. These studies reinforce the importance of leveraging both structural information and uncertainty quantification when solving adversarial optimization problems.
In \ref{app:runtime}, we report the average per-iteration runtime for BONSAI and ARBO across all case studies to quantify the computational tradeoffs introduced by function network-based modeling. While BONSAI is moderately more expensive, this overhead remains negligible relative to the cost of real-world simulations where each function evaluation may take several minutes to hours or longer. Moreover, BONSAI’s improved robustness and sample efficiency often justify this additional cost. Future research could explore adaptive mechanisms to tune this tradeoff dynamically based on the cost of the underlying evaluation process.

\section{Conclusion}
\label{sec:conclusion}

This work introduced \textbf{BONSAI}, a new framework for robust Bayesian optimization of structured black-box systems. By representing the objective function as a function network (i.e., a directed graph of known and unknown components), BONSAI effectively leverages partial structural knowledge commonly available in engineering problems (especially those involving complex, high-fidelity simulators). This ``structure-aware'' surrogate modeling enables more accurate uncertainty quantification and more targeted acquisition strategies for robust optimization under limited evaluation budgets.
Our numerical results across eight case studies -- including both synthetic and real-world engineering problems -- demonstrate that BONSAI consistently outperforms existing robust BO methods such as ARBO. The most significant gains are observed in problems with hierarchical or modular structures, where intermediate outputs can be meaningfully reused during learning. Ablation studies further show that both the structure-aware surrogate and the proposed two-stage (max-min) Thompson sampling acquisition strategy substantially contribute to BONSAI’s improved efficiency and robustness.

Looking forward, several important directions remain for future work. First, the current BONSAI framework is heuristic and lacks theoretical guarantees; analyzing its convergence behavior and establishing robust regret bounds could provide valuable insights into its performance limits. Second, many real-world applications require balancing multiple objectives under uncertainty, motivating the need to extend BONSAI to robust multi-objective optimization settings. While our recent work has demonstrated how to perform efficient multi-objective optimization over function networks \cite{MOBONS}, combining this with adversarial robustness introduces new (significant) computational and algorithmic challenges that remain open. Third, many simulation-based design problems involve hard or soft constraints. Adapting the BONSAI framework to handle constraints -- by modeling them using parallel function networks and incorporating them into the acquisition process via penalty-based techniques, as in the CARBO method \cite{Kudva2022} -- would significantly broaden its practical utility. Fourth, this work assumes that the uncertainty set is fixed and known. Exploring how robust solutions evolve as the uncertainty set is expanded or perturbed could offer insights into the sensitivity of recommended robust optimal designs. Finally, BONSAI currently relies on standard GP priors for black-box nodes. Investigating alternative covariance function choices and developing adaptive strategies to modify the GP prior based on the structure or dimensionality of each node could help improve performance in high-dimensional or data-scarce settings.

\section*{Acknowledgment}

The authors gratefully acknowledge support from the National Science Foundation under Grant \#2237616.

%% References
\bibliographystyle{elsarticle-num}           % Include this if you use bibtex
\bibliography{references}

\appendix
% \numberwithin{equation}{section}
% \numberwithin{figure}{section}

\section{Illustration of BONSAI}
\label{app:bonsai}

\setcounter{figure}{0}
\renewcommand{\thefigure}{A\arabic{figure}}

To provide further intuition into how BONSAI operates, Figure~\ref{fig:bonsai_steps} illustrates its progression on the two-dimensional white-box test case in Figure~\ref{fig:grey_box}. The top row visualizes the true function $g(x, w)$, with contour lines plotted over the design variable $x$ and uncertainty $w$. The magenta star denotes the global robust solution, defined as the solution to $\max_x \min_w g(x,w)$, and the red vertical line indicates the robust design that would be selected given perfect knowledge of the system. The lower panels depict how BONSAI incrementally refines its understanding of the function network and its predictions across nine optimization iterations. Each row corresponds to one iteration. The left column shows the posterior mean of $g(x, w)$ based on the GP-based function network, where previously sampled points are shown in black and the current query point is highlighted with a magenta diamond. The center column displays the Thompson sample realization of the lower-level objective $\widehat{g}_t^{(x)}(x, w)$, which is used to identify the next design point by solving the max-min problem. The right column plots the mean and confidence bounds for the solution to the inner problem $\min_w g(x, w)$ derived from the posterior for $g$. The green vertical line indicates the current recommended design based on the surrogate model using the recommender in \eqref{eq:recommender}. As the BONSAI proceeds, it rapidly converges to the true robust solution.

\begin{figure}[htb!]
  \centering
  \includegraphics[width=0.75\textwidth]{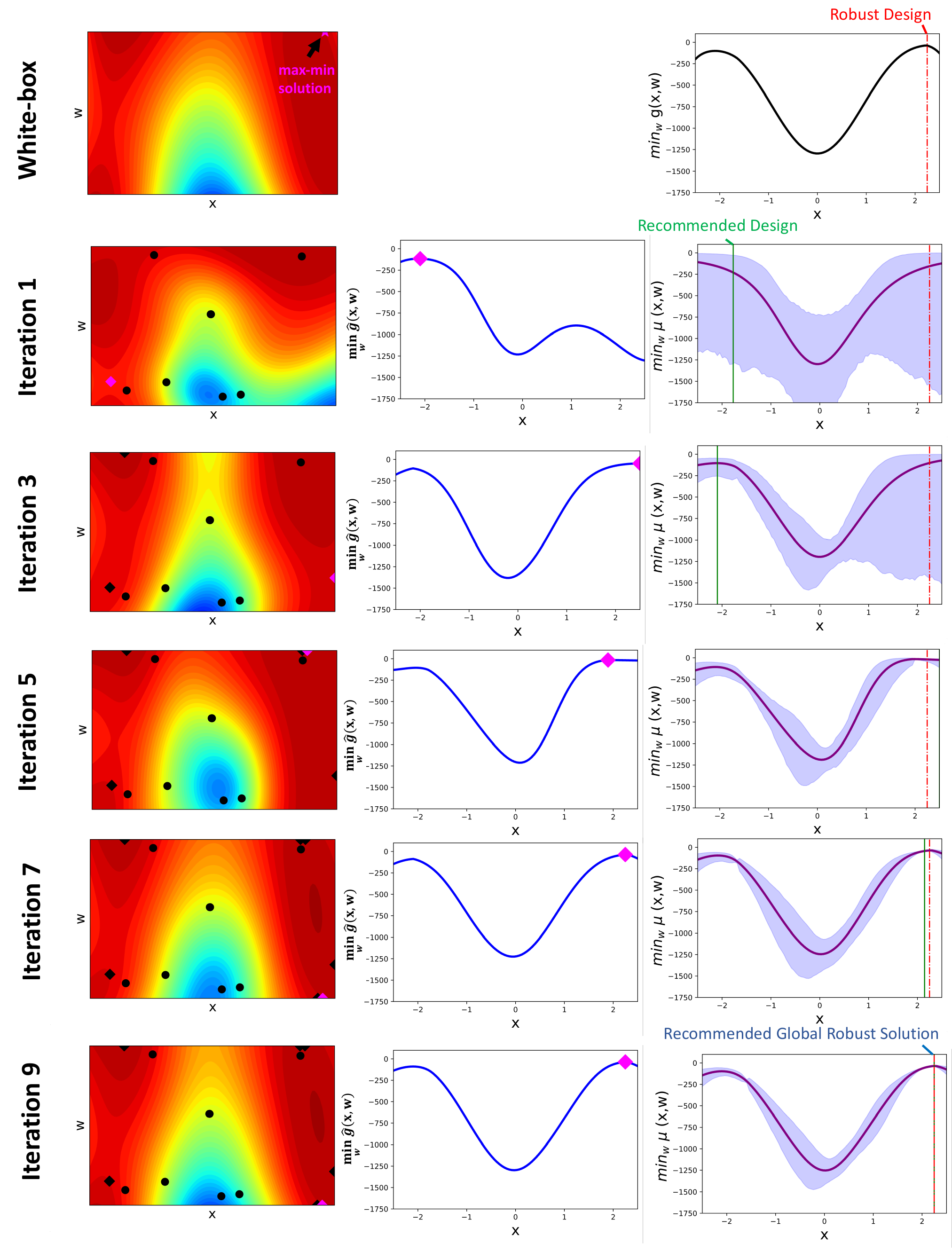}
  \caption{
    Step-by-step progression of BONSAI (Algorithm~\ref{alg:bonsai}) applied to an illustrative example problem.
    \textbf{Top row:} Contours of the true objective $g(x, w)$, as in Figure~\ref{fig:grey_box}. The magenta star denotes the global robust solution $x^\star$, and the red vertical line marks the robust design obtained with perfect knowledge of $g$.
    \textbf{Iterations 1-9:} Visualization of BONSAI across optimization steps.
    \textbf{Left column:} Posterior mean of the network GP surrogate for $g(x,w)$, with all past evaluation points (black dots) and current query (magenta diamond).
    \textbf{Center column:} Thompson-sampled inner objective \( \widehat{g}_t^{(x)}(x, w) \) used for robust design selection.
    \textbf{Right column:} Predicted solution to inner problem $\min_w g(x, w)$ using the posterior for $g$ given all past data and function network structure. The green vertical line indicates the recommended design at each step. By iteration 9, the recommendation matches the global robust optimum.
  }
  \label{fig:bonsai_steps}
\end{figure}

\section{Theoretical Regret Analysis for Nominal Case}
\label{app:theoretical-conv}

In this appendix, we analyze the Bayesian regret \cite{russo2014learning, pmlr-v84-kandasamy18a, takeno2023posterior} of a simplified version of BONSAI in which the uncertainty set $\mathcal{W}$ reduces to a singleton (so no inner minimization over $\bs{w}$ is required). While this corresponds to the nominal (non-robust) setting, it captures an essential difficulty of optimization over function networks. Establishing guarantees in this case provides insight into the performance of BONSAI and ensures that the method is theoretically well-grounded. To our knowledge, this is the first extension of Thompson sampling (TS) regret analysis from the classical black-box setting \cite{takeno2023posterior} to the more general function network framework. Moreover, we show that the results extend beyond directed acyclic graphs to cyclic function networks, thereby providing the first finite-time regret bounds (and associated convergence guarantees) for this broader and practically important class of models.

\subsection{Problem setup and notation}

We are interested in analyzing the nominal version of \eqref{eq:max-min} where $g$ is an expensive-to-evaluate objective function that only depends on design variables $\bs{x}$. Formally, we want to maximizing this function, i.e., identify $\bs{x}^\star = \argmax_{\bs{x} \in \mathcal{X}} g(\bs{x})$ over the design/input space $\mathcal{X} \subset \mathbb{R}^{D}$. We assume that $g$ is defined in terms of a ``function network'' such as \eqref{eq:network-functions}. To simplify our initial analysis, we will assume this network is a directed acyclic graph (DAG) with nodes $k=1,\dots,K$ in topological order:
\begin{align} \label{eq:fn-nom}
    h_k(\bs{x}) &= f_k\left( \bs{x}_{I(k)}, \{ h_j(\bs{x}) \}_{j \in J(k)} \right), \qquad k = 1, \ldots, K,
\end{align}
with $g(\bs{x}) = h_K(\bs{x})$, where $I(k) \subseteq \{ 1, \ldots, D \}$ is the subset of the design variables affecting node $k$ and $J(k) \subseteq \{ 1, \ldots, K \} \setminus \{ k \}$ is the subset of other node outputs that affect node $k$. Due to the acyclic structure of the graph, we can evaluate each function in the network as the values of its parent nodes become available, meaning that we can always evaluate \eqref{eq:fn-nom} through a forward iterative process. 

One evaluation of $g$ at a specific $\bs{x}$ reveals all node outputs at the induced inputs, i.e., the ``full network observation'' case. We focus on the Bayesian setting where all nodes are sample paths from \textit{independent} Gaussian processes (GPs) with zero mean and stationary covariance (or kernel) functions $\Sigma_{0,k}$, i.e., $f_k \sim \mathcal{GP}( 0, \Sigma_{0,k} )$ for all $k = 1,\ldots, K$. Thus, at each iteration $t$, our algorithm decides the next query point $\bs{x}_t$ at which we evaluate the function network based on the Bayesian model and our choice of acquisition function. We assume that the observations at each node $y_{t,k} = h_k( \bs{x}_t ) + \epsilon_{t, k} = f_k( \bs{z}_{t, k} ) + \epsilon_{t, k}$ is contaminated by Gaussian noise $\epsilon_{t, k} \sim \mathcal{N}( 0, \sigma^2 )$ where $\bs{z}_{t,k} = ( \bs{x}_{t, I(k)}, \{ h_j( \bs{x}_t ) \}_{j \in J(k)} )$ is our shorthand for the full set of inputs to unknown node function $f_k$ at iteration $t$. Let $\mathcal{D}_{t-1} = \cup_{k=1}^K \mathcal{D}_{t-1, k}$ be the complete set of data available at the beginning of the $t$-th iteration, with $\mathcal{D}_{t-1, k} = \{ (\bs{z}_{i,k}, y_{i,k}) \}_{i=1}^{ n_{t-1} }$ and $n_{t-1} > 0$. Then, the posterior distribution $p( f_k | \mathcal{D}_{t-1} ) = p( f_k | \mathcal{D}_{t-1, k} )$ (due to independence) remains a GP for each $k = 1,\ldots, K$, whose posterior mean $\mu_{t-1, k}(\bs{x})$ and variance $\sigma^2_{t-1, k}(\bs{x})$ functions can be derived in closed-form using standard Gaussian math (similar to \eqref{eq:posterior-gp}). It is important to note that the functions $h_k$ are in general not GPs (as they are compositions of GPs) such that $g$ is generally not a GP. 

Even though the prediction of $g(\bs{x}) | \mathcal{D}_{t-1}$ is not Gaussian, it is a random variable with a well-defined mean $\mu^{(g)}_{t-1}( \bs{x} )$ and standard deviation $\sigma^{(g)}_{t-1}( \bs{x} )$ for any inference point $\bs{x}$, which will be needed in our subsequent analysis. 

The bounds that we derive will depend on the \textit{maximum information gain} (MIG) of the node functions -- a fundamental quantity in the Bayesian experimental design literature \cite{srinivas2010gaussian}. MIG provides a measure of informativeness of any finite set of sampling points, and can be defined as follows:

\begin{definition}
    Let $f_k \sim \mathcal{GP}(0, \Sigma_{0,k})$ and $\mathcal{A} \subset \mathcal{X}$ be any subset of points sampled from $\mathcal{X}$. Then, the MIG after $T$ observations $\gamma_{T,k}$ of $f_k$ is defined as:
    \begin{align*}
        \gamma_{T, k} = \max_{\mathcal{A} \subset \mathcal{X}, |\mathcal{A}| = T} I( \bs{y}_{\mathcal{A}, k} ; \bs{f}_{\mathcal{A}, k} ),
    \end{align*}
    where $I$ denotes the (Shannon) mutual information of two random variables and $\bs{y}_{\mathcal{A}, k}$ and $\bs{f}_{\mathcal{A}, k}$ are, respectively, the collection of noisy observations and outcomes of the node function $f_k$ evaluated at the points in $\mathcal{A}$.
\end{definition}

Due to the Gaussian nature of the node functions, the mutual information admits a simple closed-form expression:
\begin{align*}
    I( \bs{y}_{\mathcal{A}, k} ; \bs{f}_{\mathcal{A}, k} ) = \frac{1}{2}
    \log\det \left( \bs{I} + \sigma^{-2} \bs{K}_{\mathcal{A}, k} \right),
\end{align*}
where $\bs{I}$ denotes the identity matrix and $\bs{K}_{\mathcal{A}, k} = [\Sigma_{0,k}(\bs{x}, \bs{x}')]_{\bs{x}, \bs{x}' \in \mathcal{A}}$ is the prior covariance matrix between all points in $\mathcal{A}$. It is known that the MIG $\gamma_{T,k}$ is sublinear in $T$ for many commonly used kernel functions including the radial basis function (RBF) and Matern-$\nu$ kernels \cite{srinivas2010gaussian, vakili2021information}.

\subsection{Performance measure}

We focus on Bayesian cumulative regret (BCR) and Bayesian simple regret (BSR) over a number of iterations $T$, which are defined as follows
\begin{align}
    \text{BCR}_T &= \mathbb{E} \left\lbrace \sum_{t=1}^T g(\bs{x}^\star) - g(\bs{x}_t) \right\rbrace, \\
    \text{BSR}_T &= \mathbb{E} \left\lbrace g(\bs{x}^\star) - \max_{ t=1,\ldots, T} g(\bs{x}_t) \right\rbrace,
\end{align}
where the expectation $\mathbb{E}\{ \cdot \}$ is taken over all randomness, i.e., $\{ f_k \}_{k=1}^K$, $\{ \epsilon_{t,k} \}_{t \geq 1, 1 \leq k \leq K}$, and the randomness of the optimization algorithm. Since we cannot identify $\argmax_{t = 1, \ldots, T} g(\bs{x}_t)$ in practice when we have noisy observations, it is also popular to analyze BSR with recommendation. The recommender in the nominal case is typically $\hat{\bs{x}}_T = \argmax_{\bs{x} \in \mathcal{X}} \mu^{(g)}_{T-1}( \bs{x} )$\footnote{We use the standard recommender that uses the posterior at the start of the last iteration for simplicity; this can easily be extended to incorporate the final observation $\hat{\bs{x}}_T^+ = \argmax_{\bs{x} \in \mathcal{X}} \mu^{(g)}_{T}( \bs{x} )$ by incorporating a final hypothetical TS step that leads to a slight further reduction in the modified BSR.}, leading to the modified BSR of the form $\overline{\text{BSR}}_T = \mathbb{E}\{ g(\bs{x}^\star) - g( \hat{\bs{x}}_T ) \}$. 

We can minorly extend \cite[Lemma 2.1]{takeno2023posterior} to show an important relationship between BSR, BCR, and the modified BSR (with recommendation).

\begin{lemma} \label{lem:BSRvsBCR}
    For any sequential sampling procedure and (measurable) recommender $\hat{\bs{x}}_t = \argmax_{\bs{x} \in \mathcal{X}} \mu^{(g)}_{t-1}( \bs{x} )$ at every iteration $t$, 
    \begin{align*}
        \overline{\textnormal{BSR}}_T \leq \frac{1}{T} \sum_{t=1}^T \overline{\textnormal{BSR}}_t \leq \frac{1}{T} \textnormal{BCR}_T.
    \end{align*}
\end{lemma}

The proof is provided in \ref{subsec:prooflem1} (at the end of this section) since the details are not important for our main message. This lemma, however, is important since it shows, if we can prove BCR grows sublinearly with $T$, then it immediately follows $\overline{\text{BSR}}_T \leq \text{BCR}_T / T \to 0$ such that the expected optimality gap for the recommender's value goes to 0 as $T \to \infty$. By Markov's inequality, we see that $\mathbb{P}\{ g(\bs{x}^\star) - g(\hat{\bs{x}}_T) \geq \varepsilon \} \leq \overline{\text{BSR}}_T / \varepsilon \to 0$ for all $\varepsilon > 0$ (i.e., convergence in probability). Therefore, we focus on showing this sublinearity property under reasonably mild conditions in the rest of this appendix.

\subsection{Algorithm: Thompson sampling for function networks}

We focus on a special case of BONSAI (Algorithm \ref{alg:bonsai}) where we ignore the uncertainty. In this case, we simply generate a sample path (or TS) of the unknown objective $g$ from the posterior $g_t \sim p( g \mid \mathcal{D}_{t-1} )$ at the start of the $t$-th iteration and then select $\bs{x}_t \leftarrow \argmax_{ \bs{x} \in \mathcal{X} } g_t(\bs{x})$. We then query the network at $\bs{x}_t$, update the posteriors, and repeat this process.

It is easy to generate this posterior sample by simply drawing independent realizations of $f_k$ from their posteriors and substituting them into \eqref{eq:fn-nom}.

\subsection{Assumptions}

Our analysis requires some assumptions on the problem setup. Note that we have stated some of these previously, but do our best to clearly and exhaustively list them all here for ease of reference. 

\begin{description}
    \item[(A1) Finite design.] The design space $\mathcal{X}$ has finite cardinality $|\mathcal{X}| < \infty$. 
    \item[(A2) Independent GPs and Gaussian noise.] Each node $f_k$ has an independent GP prior with zero mean; node observations have i.i.d. Gaussian noise $\mathcal{N}(0, \sigma^2)$ with $\sigma^2 > 0$.
    \item[(A3) DAG with full network observations.] Querying the network at $\bs{x}$ reveals $g(\bs{x})$ and all node outputs $h_k(\bs{x})$ at the inputs induced by the DAG recursion defined in \eqref{eq:fn-nom}. 
    \item[(A4) Network regularity condition.] For each node $k$ and iteration $t$, the posterior mean is Lipschitz continuous in its parent vector with per-parent constants:
    \begin{align*}
        | \mu_{t-1, k}( \bs{x}_{I(k)}, \bs{u} ) - \mu_{t-1, k}( \bs{x}_{I(k)}, \bs{v} ) | \leq \sum_{j \in J(k)} L_{k \leftarrow j} | u_j - v_j |, ~~~ \forall (\bs{u}, \bs{v}). 
    \end{align*}
    \item[(A5) Normalized sub-Gaussian marginals.] There exists $\kappa \geq 1$ such that for all $t$ and $\bs{x} \in \mathcal{X}$, the following holds:
    \begin{align*}
        \mathbb{E}\left\lbrace \exp\left( \lambda \frac{g_t(\bs{x}) - \mu_{t-1}^{(g)}(\bs{x})}{\sigma_{t-1}^{(g)}(\bs{x})} \right) \mid \mathcal{D}_{t-1} \right\rbrace \leq \exp\left( \frac{\kappa^2 \lambda^2}{2} \right), ~~~ \forall \lambda \in \mathbb{R}
    \end{align*}
    \item[(A6) Measurable argmax function.] The argmax selection rule $\pi(f)$, which returns a maximizer $\pi(f) \in \argmax_{\bs{x} \in \mathcal{X}} f(\bs{x})$ of any function $f : \mathcal{X} \to \mathbb{R}$ using a deterministic tie-breaking rule, is measurable. 
\end{description}

(A1) is made for simplicity and can be relaxed to continuous spaces using an approach like that shown in \cite{srinivas2010gaussian}, which makes some additional assumptions about the smoothness of the sample paths (which are again valid for many common kernels). 

(A2) is standard and can be easily modified to allow for different sub-Gaussian distributions and/or different noise variances per node (with only slight changes in the derived constants). 

(A3) matches the original work that introduced BOFN \cite{BOFN}. It is worth noting that other forms of function networks exist, e.g., \cite{BOFN_PE} and we expect the many of the ideas used here can be extended to those settings (but we save that for future work). We relax the DAG assumption in the final subsection, as we show the relevant proofs carry over with a minor change to the constant in our derived BCR bound. 

(A4) is a key assumption that ensures the network behaves in a reasonable fashion. A sufficient condition for this to hold is sample-path regularity of $f_k$: if the sample paths are $C^1$ with almost surely bounded parent-derivatives, then the posterior mean function inherits uniform Lipschitz constants by the dominated convergence theorem. 

(A5) states that the standardized Thompson sample/draw of the scalar objective (output of the function network) is uniformly sub-Gaussian; this means that we can find a constant $\kappa \geq 1$ that does not depend on $t$ or $\bs{x}$, which will help us obtain the BCR bound. Note that in the single-node GP case (standard black-box BO setting), $\kappa = 1$. A sufficient condition for (A5) to hold is that the function mapping $\Phi_{t, \bs{x}}$ from a standard normal $U \sim \mathcal{N}(0,\bs{I})$ to scalar objective $g_t(\bs{x}) = \Phi_{t,\bs{x}}(U)$ is globally Lipschitz continuous in $U$ with constant $L_{t, \bs{x}}$ (constructed through whitening transformation across the nodes). We can then show $\kappa \leq \sup_{t, \bs{x}} L_{t,\bs{x}} / \sigma^{(g)}_{t-1}(\bs{x}) < \infty$. This is intuitively reasonable because, for any fixed $t\geq 1$ and $\bs{x} \in \mathcal{X}$, $g_t(\bs{x})$ is obtained by recursively composing node-level Gaussian samples through a DAG. Although the resulting scalar is generally \textit{non-Gaussian}, it is a smooth function of finitely many underlying Gaussian innovations; under mild smoothness, Gaussian concentration implies sub-Gaussian tails for such transforms.

(A6) is needed to avoid any pathological issues that could arise from optimizing over complex functions or spaces; lack of measurability will result in the expectations in our regret analysis to not be well-defined.

\subsection{Network sensitivity}

We need to account for the impact the network can have on the node GPs as they are propagated through the DAG. We first define the sum of squares of edge Lipschitz constants for node $k$:
\begin{align*}
    \bar{L}_k^2 = \sum_{j \in J(k)} L_{k \leftarrow j}^2.
\end{align*}
We then define the parent-sensitivity matrix $\bs{A} \in \mathbb{R}^{K \times K}$ with elements:
\begin{align*}
    A_{kj} = \begin{cases}
        \bar{L}_k^2, &\text{if } j \in J(k), \\
        0, &\text{otherwise}.
    \end{cases}
\end{align*}
Because the graph is assumed to be acyclic, $\bs{A}$ is a strictly lower triangular matrix such that $( \bs{I} - \bs{A} )^{-1} = \bs{I} + \bs{A} + \cdots + \bs{A}^{K-1}$ with non-negative entries. We also define an overall network amplification constant:
\begin{align*}
    L_{\mathrm{net}} = \| \bs{e}_K^\top ( \bs{I} - \bs{A} )^{-1} \|_\infty,
\end{align*}
where $\bs{e}_K \in \mathbb{R}^K$ is a vector whose last element is 1 and all others are 0. 
% We will see that $L_{\mathrm{net}}$ is the largest total path gain from any single upstream node's variance to the output. It captures how much uncertainty can be amplified by the network structure before it reaches the final node: $g(\bs{x}) = h_K(\bs{x})$.

\subsection{Supporting lemmas}

We start by proving the following lemma that relates the individual terms in BCR to useful quantities:

\begin{lemma}[Posterior-sampling identity] \label{lem:psi}
    Let $\bs{x}^\star \in \argmax_{\bs{x} \in \mathcal{X}} g(\bs{x})$ and $\bs{x}_t = \pi( \bs{g}_t ) = \argmax_{\bs{x} \in \mathcal{X}} g_t(\bs{x})$ for $\bs{g}_t \sim p( \bs{g} \mid \mathcal{D}_{t-1} )$. Then, the following equalities hold
    \begin{align*}
        \mathbb{E}\left\lbrace g(\bs{x}^\star)-g(\bs{x}_t) \mid \mathcal{D}_{t-1} \right\rbrace &= \mathbb{E}\left\lbrace g_t(\bs{x}_t)-g(\bs{x}_t) \mid \mathcal{D}_{t-1} \right\rbrace, \\
        &= \mathbb{E}\left\lbrace g_t(\bs{x}_t)-\mu_{t-1}^{(g)}(\bs{x}_t) \mid \mathcal{D}_{t-1} \right\rbrace.
    \end{align*}
    The same equalities also hold unconditionally (i.e., with total $\mathbb{E}\{ \cdot \}$). 
\end{lemma}

\begin{proof}
    Let $g$ and $g'$ be i.i.d. sample paths from the posterior; by symmetry, $\mathbb{E}\left\lbrace \max_{\bs{x}} g(\bs{x}) \mid \mathcal{D}_{t-1} \right\rbrace = \mathbb{E}\left\lbrace \max_{\bs{x}} g'(\bs{x}) \mid \mathcal{D}_{t-1} \right\rbrace$. Since $g(\bs{x}^\star) = \max_{\bs{x}} g(\bs{x})$ and $g_t(\bs{x}_t) = \max_{\bs{x}} g_t(\bs{x})$ and measurability holding by (A6), we know that $\mathbb{E}\left\lbrace g(\bs{x}^\star) \mid \mathcal{D}_{t-1} \right\rbrace = \mathbb{E}\left\lbrace g_t(\bs{x}_t) \mid \mathcal{D}_{t-1} \right\rbrace$. The two stated relationships then follow by subtracting $\mathbb{E}\left\lbrace g(\bs{x}_t) \mid \mathcal{D}_{t-1} \right\rbrace = \mathbb{E}\{ \mu^{(g)}_{t-1}(\bs{x}_t) \mid \mathcal{D}_{t-1} \}$ from this equality. The unconditional version then follows from the tower property of expectations. 
\end{proof}

% $\mathbb{E}\left\lbrace  \right\rbrace$
% $\mathbb{E}\left\lbrace  \mid \mathcal{D}_{t-1} \right\rbrace$

Next, we develop bounds on the variance of the nodes of the network and relate them to the (conditional) predictive variance terms of the individual node functions. Intuitively, we can think of this as finding a way to bound an important property of interest (the total variance of the objective) in terms of quantities that we know how to bound using the MIG. 

\begin{lemma}[Variance recursion and network bound] \label{lem:varrec}
    For iteration $t$, define the following quantities:
    \begin{align*}
        V_k(t) &= \textnormal{Var}_{t-1}( h_k(\bs{x}_t) ), \\
        \bs{Z}_{t,k} &= ( \bs{x}_{t, I(k)}, \{ h_j(\bs{x}_t) \}_{j \in J(k)} ), \\
        S_k(t) &= \mathbb{E}_{t-1}\left\lbrace \sigma^2_{t-1,k}( \bs{Z}_{t,k} ) \right\rbrace,
    \end{align*}
    where the subscript $t-1$ implies being conditioned on $\mathcal{D}_{t-1}$. Then, for each node $k$, the following inequality holds:
    \begin{align*}
        V_k(t) \leq S_k(t) + \bar{L}_k^2 \sum_{j \in J(k)} V_j(t).
    \end{align*}
    We can equivalently state this in matrix-vector form: $\bs{V}(t) \leq \bs{S}(t) + \bs{A} \bs{V}(t)$ where $\bs{Q}(t) = [Q_1(t), \ldots, Q_K(t)]^\top$ for $Q \in \{ V, S \}$, hence:
    \begin{align*}
        \bs{V}(t) \leq (\bs{I} - \bs{A})^{-1} \bs{S}(t). 
    \end{align*}
    Thus, for the objective node $K$, we specifically have:
    \begin{align*}
        ( \sigma^{(g)}_{t-1}(\bs{x}_t) )^2 = \textnormal{Var}_{t-1}( g(\bs{x}_t) ) & = V_K(t) \leq \bs{e}_K^\top (\bs{I} - \bs{A})^{-1} \bs{S}(t) \leq L_\textnormal{net} \sum_{k=1}^K S_k(t).
    \end{align*}
\end{lemma}

\begin{proof}
    We start with the law of total variance:
    \begin{align*}
        & \textnormal{Var}_{t-1}( h_k(\bs{x}_t) ) = \mathbb{E}_{t-1}\{ \textnormal{Var}_{t-1}( h_k(\bs{x}_t) \mid \bs{Z}_{t,k} ) \} + \text{Var}_{t-1}\left( \mathbb{E}_{t-1}\{ h_k(\bs{x}_t) \mid \bs{Z}_{t,k} \} \right),
    \end{align*}
    where the outer operators must also marginalize over the randomness of $\bs{Z}_{t,k}$.  
    The first term is exactly equal to $S_k(t)$ since, conditioning on the random inputs, we get that $\textnormal{Var}_{t-1}( h_k(\bs{x}_t) \mid \bs{Z}_{t,k} ) = \sigma_{t-1, k}^2( \bs{Z}_{t,k} )$. The second term is a little more complicated and so we instead look to bound it using the Lipschitz regularity condition in (A4). We first notice that:
    \begin{align*}
        \text{Var}_{t-1}\left( \mathbb{E}_{t-1}\{ h_k(\bs{x}_t) \mid \bs{Z}_{t,k} \} \right) = \text{Var}_{t-1}\left( \mu_{t-1,k}( \bs{Z}_{t,k} ) \right). 
    \end{align*}
    We write out $\bs{Z}_{t,k} = ( \bs{x}_{t, I(k)}, \bs{Y}_{t,k} )$ and define $\phi_k( \bs{y} ) = \mu_{t-1,k}( \bs{x}_{t, I(k)}, \bs{y} )$ as the posterior mean function for node $k$ as a function of only of the second argument (other network node outputs). By (A4), for any $\bs{y}$ and $\bs{y}'$, we have:
    \begin{align*}
        | \phi_k( \bs{y} ) - \phi_k( \bs{y}' ) | \leq \sum_{j \in J(k)} L_{k \leftarrow j} | y_j - y'_j |.
    \end{align*}
    Let $\bs{Y}_{t,k}'$ be an independent copy of $\bs{Y}_{t,k}$ conditional on $\mathcal{D}_{t-1}$ (i.e., another independent draw of parents at the same $\bs{x}_t$). The Efron-Stein identity gives:
    \begin{align*}
        \text{Var}_{t-1}( \phi_k( \bs{Y}_{t,k} ) ) = \frac{1}{2}\mathbb{E}_{t-1}\left\lbrace ( \phi_k( \bs{Y}_{t,k} ) - \phi_k( \bs{Y}_{t,k}' ) )^2 \right\rbrace.
    \end{align*}
    Applying the Lipschitz bound and Cauchy-Schwarz inequality, we can derive:
    \begin{align*}
        ( \phi_k( \bs{Y} ) - \phi_k( \bs{Y}' )^2 ) \leq \underbrace{\left( \sum_{j \in J(k)} L_{k \leftarrow j}^2 \right)}_{\bar{L}_k^2} \sum_{j \in J(k)} (Y_j - Y_j')^2.
    \end{align*}
    Taking the expectation $\mathbb{E}_{t-1}\{ \cdot \}$ and combining with previous expressions:
    \begin{align*}
        \text{Var}_{t-1}\left( \mu_{t-1,k}( \bs{Z}_{t,k} ) \right) &= \text{Var}_{t-1}( \phi_k( \bs{Y}_{t,k} ) ), \\
        & \leq \frac{1}{2} \bar{L}_k^2 \sum_{j \in J(k)}\mathbb{E}_{t-1}\left\lbrace ( [\bs{Y}_{t,k}]_j  - [\bs{Y}'_{t,k}]_j )^2  \right\rbrace,  \\
        & = \frac{1}{2} \bar{L}_k^2 \sum_{j \in J(k)} 2\text{Var}_{t-1}\left( h_j( \bs{x}_t ) \right), \\
        & = \bar{L}_k^2 \sum_{j \in J(k)} V_j(t).
    \end{align*}
    Going back to the law of total variance, we now see:
    \begin{align*}
        V_k(t) &= S_k(t) + \text{Var}_{t-1}\left( \mu_{t-1,k}( \bs{Z}_{t,k} ) \right) \leq S_k(t) + \bar{L}_k^2 \sum_{j \in J(k)} V_j(t).
    \end{align*}
    holds for every node $k = 1,\ldots,K$. 
    The matrix-vector form immediately follows. Because the graph is DAG, $\bs{A}$ is strictly lower triangular, so that $\bs{I} - \bs{A}$ is invertible. The final inequalities follow from the fact that $V_K(t) = \bs{e}_K^\top \bs{V}(t)$, $\bs{V}(t) \leq (\bs{I} - \bs{A})^{-1}\bs{S}(t)$, and Holder's inequality $V_K(t) \leq \bs{e}_K^\top (\bs{I} - \bs{A})^{-1}\bs{S}(t) \leq \| \bs{e}_K^\top (\bs{I} - \bs{A})^{-1} \|_\infty \| \bs{S}(t) \|_1 = L_\text{net} \| \bs{S}(t) \|_1$. 
\end{proof}

We now want to relate the variance sum that we derived in the previous lemma to the MIGs of the node functions (defined previously), which we do in the following lemma:

\begin{lemma}[Node variance sum $\leq$ MIG] \label{lem:mig}
    For each node $k$, we have:
    \begin{align*}
        \sum_{t=1}^T \mathbb{E}\{ S_k(t) \} \leq C_1 \gamma_{T, k}
    \end{align*}
    where $S_k(t) = \mathbb{E}_{t-1}\left\lbrace \sigma^2_{t-1,k}( \bs{Z}_{t,k} ) \right\rbrace$ (same as Lemma \ref{lem:varrec}) and $C_1 = 2/\log(1+\sigma^{-2})$ is a constant. 
\end{lemma}

\begin{proof}
    By \cite[Lemma 5.4]{srinivas2010gaussian}, we know that, for the GP at node $k$,
    \begin{align*}
        \sum_{t=1}^T\sigma^2_{t-1,k}( \bs{z}_{t,k} ) \leq C_1 \gamma_{T,k},
    \end{align*}
    for any realized sequence of inputs $\bs{z}_{1,k}, \ldots, \bs{z}_{T,k}$. Since $\gamma_{T,k}$ is a deterministic constant, we can plug in random sequence $\bs{Z}_{1,k}, \ldots, \bs{Z}_{T,k}$ and take the expectation on both sides to yield the claimed inequality. 
\end{proof}

The final lemma that we prove next enables us to relate the key difference quantity in Lemma \ref{lem:psi} to the standard deviation of $g$ (effectively yielding an optimistic bound on the Bayesian regret at any fixed iteration $t$):

\begin{lemma}[Finite-set optimism bound] \label{lem:optimism}
    Let $Y_{t, \bs{x}} = [ g_t(\bs{x}) - \mu^{(g)}_{t-1}(\bs{x}) ] / \sigma^{(g)}_{t-1}(\bs{x})$ and $M_t = \max_{\bs{x} \in \mathcal{X}} Y_{t,\bs{x}}$. Under the sub-Gaussian assumption (A5), we have:
    \begin{align*}
        \mathbb{E}\{ (M_t)_+^2 \} \leq 2 \kappa^2 ( 1 + \log| \mathcal{X}| ),
    \end{align*}
    where $(M_t)_+ = \max\{ M_t, 0 \}$. Consequently, the following also holds:
    \begin{align*}
        \mathbb{E}\{ g_t(\bs{x}_t) - \mu^{(g)}_{t-1}(\bs{x}_t) \} \leq \sqrt{2 \kappa^2 ( 1 + \log| \mathcal{X}| )} \sqrt{ \mathbb{E}\{ ( \sigma^{(g)}_{t-1}(\bs{x}_t) )^2 \} }
    \end{align*}
\end{lemma}

\begin{proof}
    By (A5), for each $\bs{x}$ given $\mathcal{D}_{t-1}$, we have:
    \begin{align*}
        \mathbb{P}\{ Y_{t, \bs{x}} \geq u \mid \mathcal{D}_{t-1} \} \leq \exp\left( -\frac{u^2}{2 \kappa^2}  \right), ~~~ u \geq 0.
    \end{align*}
    By combining this with a union bound over the finite set $\mathcal{X}$, we get:
    \begin{align*}
        \mathbb{P}\{ M_t \geq u \mid \mathcal{D}_{t-1} \} &= \mathbb{P}\{ \exists \bs{x} \in \mathcal{X} : Y_{t, \bs{x}} \geq u \mid \mathcal{D}_{t-1} \}, \\
        & \leq \sum_{ \bs{x} \in \mathcal{X} } \mathbb{P}\{ Y_{t, \bs{x}} \geq u \mid \mathcal{D}_{t-1} \}, \\
        & \leq | \mathcal{X} | \exp\left( -\frac{u^2}{2 \kappa^2}  \right).
    \end{align*}
    Integrating the tail bound using standard methods gives the following conditional second-moment bound:
    \begin{align*}
        \mathbb{E}_{t-1}\{ (M_t)_+^2 \} \leq 2 \kappa^2 \log(1 + |\mathcal{X}|).
    \end{align*}
    Taking the total expectation using the tower property again shows that the unconditional version also holds, as claimed. 
    
    We can relate the one-step ``optimism'' to $M_t$ (to derive the second inequality) as follows. Start with the definition of $Y_{t, \bs{x}}$:
    \begin{align*}
       g_t(\bs{x}) - \mu^{(g)}_{t-1}(\bs{x}) = \sigma^{(g)}_{t-1}(\bs{x}) Y_{t, \bs{x}}.
    \end{align*}
    Evaluating at the chosen point $\bs{x}_t$ and using $Y_{t, \bs{x}_t} \leq (M_t)_+$ (true even when $M_t < 0$), we have the pointwise inequality:
    \begin{align*}
        g_t(\bs{x}_t) - \mu^{(g)}_{t-1}(\bs{x}_t) \leq \sigma^{(g)}_{t-1}(\bs{x}) (M_t)_+.
    \end{align*}
    Taking the unconditional expectation and applying Cauchy-Schwarz yields:
    \begin{align*}
        \mathbb{E}\{ g_t(\bs{x}_t) - \mu^{(g)}_{t-1}(\bs{x}_t) \} & \leq \mathbb{E}\{ \sigma^{(g)}_{t-1}(\bs{x}) (M_t)_+ \} \\
        & \leq \sqrt{ \mathbb{E}\{ ( \sigma^{(g)}_{t-1}(\bs{x}_t) )^2 \} } \sqrt{ \mathbb{E}\{ (M_t)_+^2 \} }.
    \end{align*}
    Substituting the first yields exactly the second claimed inequality.
\end{proof}

\subsection{Main theorem for DAGs and finite $\mathcal{X}$}

We are now in a position to state and prove our main theorem that provides a novel bound on the BCR when using a TS version of BOFN (equivalent to nominal version of BONSAI).

\begin{theorem} \label{thm:bcr_bound}
    Let assumptions (A1) to (A6) hold. Then, by running the BONS algorithms with TS acquisition function (i.e., the nominal version of BONSAI; Algorithm \ref{alg:bonsai}), the BCR can be bounded as follows:
    \begin{align*}
        \textnormal{BCR}_T \leq \kappa \sqrt{2(1 + \log|\mathcal{X}|)T} \sqrt{L_\textnormal{net} C_1 \textstyle\sum_{k=1}^K \gamma_{T,k}}
    \end{align*}
    where $C_1 = 2 / \log(1 + \sigma^{-2})$.
\end{theorem}

\begin{proof}
    We start with the definition of BCR and derive a series of inequalities:
    \begin{align*}
        \text{BCR}_T &= \sum_{t=1}^T \mathbb{E} \left\lbrace g(\bs{x}^\star) - g(\bs{x}_t) \right\rbrace, \\
        &= \sum_{t=1}^T \mathbb{E} \left\lbrace g_t(\bs{x}_t)-\mu_{t-1}^{(g)}(\bs{x}_t) \right\rbrace, \\
        &\leq \sqrt{2 \kappa^2 ( 1 + \log| \mathcal{X}| )} \sum_{t=1}^T \sqrt{ \mathbb{E}\{ ( \sigma^{(g)}_{t-1}(\bs{x}_t) )^2 \} }, \\
        &\leq \sqrt{2 \kappa^2 ( 1 + \log| \mathcal{X}| )} \sqrt{T \sum_{t=1}^T \mathbb{E}\{ ( \sigma^{(g)}_{t-1}(\bs{x}_t))^2 \} }, \\
        & \leq \sqrt{2 \kappa^2 ( 1 + \log| \mathcal{X}| ) T} \sqrt{\sum_{t=1}^T L_\text{net} \sum_{k=1}^K \mathbb{E}\{ S_k(t) \} }, \\
        & \leq \sqrt{2 \kappa^2 ( 1 + \log| \mathcal{X}| ) T} \sqrt{L_\text{net} \sum_{k=1}^K C_1 \gamma_{T,k} }, 
    \end{align*}
    where the second line follows from Lemma \ref{lem:psi}, the third line follows from Lemma \ref{lem:optimism}, the fourth line follows from the Cauchy-Schwarz inequality, the fifth line follows from Lemma \ref{lem:varrec}, and the six line follows from Lemma \ref{lem:mig}. The claimed inequality then follows by rearrangement of this final expression. 
\end{proof}

As mentioned earlier, this bound has important significance for convergence; we see that the growth of the BCR has the following order with respect to $T$ and $|\mathcal{X}|$: $\text{BCR}_T = O( \sqrt{T \gamma^\text{sum}_T \log|\mathcal{X}|} )$ where $\gamma^\text{sum}_T = \sum_{k=1}^K \gamma_{T,k}$. This means, as long as $\gamma^\text{sum}_T$ is sublinaer in $T$, then $\overline{\text{BSR}}_T \to 0$ by Lemma \ref{lem:BSRvsBCR} and convergence in probability immediately follows. However, as one would expect, even just a single node being difficult to learn might compromise both performance and convergence. 

\subsection{Corollary for cyclic networks}

Although the previous analysis is only valid for DAGs, it turns out that only Lemma \ref{lem:varrec} makes use of this fact (Lemmas \ref{lem:psi}, \ref{lem:mig}, and \ref{lem:optimism} are all valid for general networks). Furthermore, Lemma \ref{lem:varrec} only minorly makes use of the DAG assumption to ensure $(\bs{I} - \bs{A})^{-1}$ exists and has non-negative entries. The identical derivation extends to cyclic networks provided two mild conditions hold: (i) for every queried $\bs{x}$ there exists a unique fixed point $\bs{h}^\star(\bs{x})$ and (ii) the matrix $\bs{A}$ satisfies $\rho( \bs{A} ) < 1$ (i.e., has a spectral radius less than one). The latter is equivalent to $(\bs{I} - \bs{A})$ being a non-singular $M$-matrix, which ensures the conditions on existence and non-negative entries. When the graph is a DAG, $\rho(\bs{A}) = 0$ whereas, if $\rho(\bs{A}) \geq 1$, the variance bound may fail due to loop amplification and thus would require an alternative proof technique (which we believe is an interesting direction for future work). 

\subsection{Proof of Lemma 1} \label{subsec:prooflem1}

We start by rewriting the modified BSR in terms of the posterior mean function for $g$ using the tower property of expectations:
\begin{align*}
    \overline{\textnormal{BSR}}_T &= \mathbb{E}\{ g(\bs{x}^\star) - g( \hat{\bs{x}}_T ) \}, \\
    &= \mathbb{E}_{\mathcal{D}_{T-1}} \left\lbrace \mathbb{E}\{ g(\bs{x}^\star) - g( \hat{\bs{x}}_T ) \mid \mathcal{D}_{T-1} \} \right\rbrace, \\
    &= \mathbb{E}_{\mathcal{D}_{T-1}} \{ \mathbb{E}\{ g(\bs{x}^\star) \mid \mathcal{D}_{T-1} \} - \mu^{(g)}_{T-1}( \hat{\bs{x}}_{T} ) \}.
\end{align*}
For every $t \leq T$, we have $\mu^{(g)}_{T-1}( \hat{\bs{x}}_{t} ) \leq \mu^{(g)}_{T-1}( \hat{\bs{x}}_{T} )$, which implies:
\begin{align*}
    \mu^{(g)}_{T-1}( \hat{\bs{x}}_{T} ) \geq \frac{1}{T}\sum_{t=1}^T \mu^{(g)}_{T-1}( \hat{\bs{x}}_{t} ).
\end{align*}
Inserting this above and rearranging, we can derive:
\begin{align*}
    \overline{\textnormal{BSR}}_T &\leq \mathbb{E}_{\mathcal{D}_{T-1}} \left\lbrace \mathbb{E}\{ g(\bs{x}^\star) | \mathcal{D}_{T-1} \} - \frac{1}{T}\sum_{t=1}^T \mu^{(g)}_{T-1}( \hat{\bs{x}}_{t} ) \right\rbrace, \\
    &= \mathbb{E}_{\mathcal{D}_{T-1}} \left\lbrace \mathbb{E} \left\lbrace g(\bs{x}^\star) - \frac{1}{T}\sum_{t=1}^T g( \hat{\bs{x}}_{t} ) ~|~ \mathcal{D}_{T-1} \right\rbrace \right\rbrace, \\
    &= \frac{1}{T} \sum_{t=1}^T \mathbb{E}\{ g(\bs{x}^\star) - g( \hat{\bs{x}}_{t} ) \}, \\
    &= \frac{1}{T} \sum_{t=1}^T \overline{\textnormal{BSR}}_t.
\end{align*}
This proves the first stated inequality. For the BCR, we note that:
\begin{align*}
    \mathbb{E}\{ g(\bs{x}^\star) - g( \hat{\bs{x}}_{t} ) \} &= \mathbb{E}_{ \mathcal{D}_{t-1} } \left\lbrace \mathbb{E}\left\lbrace g(\bs{x}^\star) - g(\bs{x}_t) + g(\bs{x}_t) - g( \hat{\bs{x}}_{t} ) \mid \mathcal{D}_{t-1} \right\rbrace  \right\rbrace, \\
    &= \mathbb{E}_{ \mathcal{D}_{t-1} } \left\lbrace \mathbb{E}\left\lbrace g(\bs{x}^\star) - g(\bs{x}_t) \mid \mathcal{D}_{t-1} \right\rbrace + \mu^{(g)}_{t-1}(\bs{x}_t) - \mu^{(g)}_{t-1}( \hat{\bs{x}}_{t} ) \right\rbrace, \\
    &\leq \mathbb{E}_{ \mathcal{D}_{t-1} } \left\lbrace \mathbb{E}\left\lbrace g(\bs{x}^\star) - g(\bs{x}_t) \mid \mathcal{D}_{t-1} \right\rbrace\right\rbrace, \\
    &= \mathbb{E} \{ g(\bs{x}^\star) - g(\bs{x}_t) \},
\end{align*}
where the third line follows from the definition of the recommender that maximizes the posterior means function such that $\mu^{(g)}_{t-1}(\bs{x}_t) \leq \mu^{(g)}_{t-1}( \hat{\bs{x}}_{t} )$ for all $t \geq 1$. Averaging over $t = 1, \ldots, T$, we get:
\begin{align*}
    \frac{1}{T} \sum_{t=1}^T \overline{\textnormal{BSR}}_t &= \frac{1}{T} \sum_{t=1}^T \mathbb{E}\{ g(\bs{x}^\star) - g( \hat{\bs{x}}_{t} ) \}, \\
    &\leq \frac{1}{T} \sum_{t=1}^T \mathbb{E} \{ g(\bs{x}^\star) - g(\bs{x}_t) \}, \\
    &= \frac{1}{T} \text{BCR}_T.
\end{align*}
This completes the proof. $\hfill \square$

\section{Description of Synthetic Test Functions}
\label{app:synthetic-test}

\setcounter{figure}{0}
\renewcommand{\thefigure}{C\arabic{figure}}

\subsection{Polynomial} \label{sec:polynomial_case}

The Polynomial test problem is a function network consisting of $K=4$ nodes. The first three are black-box (unknown) while the final one is a white-box (known) summation function. This problem involves two design variables $\bs{x}=[ x_1,x_2 ]^\top$ and two uncertainties $\bs{w}=[w_1, w_2]^\top$. The ground-truth equations can be expressed in the form of \eqref{eq:network-functions} as follows
\begin{align*}
    h_1 &= -2r_1^6 + 12.2r_1^5 - 21.2r_1 - 6.2r_1 + 6.4r_1^3 + 4.7r_1^2, \\
    h_2 &= -r_2^6 + 11r_2^5 - 43.3r_2^4 + 10r_2 + 74.8r_2^3 - 56.9r_2^2, \\
    h_3 &= 4.1r_1r_2 + 0.1r_1^2r_2^2 - 0.4r_1r_2^2 - 0.4r_1^2r_2, \\
    h_4 &= h_1 + h_2 + h_3,
\end{align*}
where $r_1 = x_1 + w_1\cos(w_2)$ and $r_2 = x_2 + w_1\sin(w_2)$. The objective $g(\bs{x},\bs{w})$ is given by \eqref{eq:network-objective} with $\bs{c} = [0,0,0,1]^\top$. The design space is $\mathcal{X} = [-0.5, 3.25] \times  [-0.5, 4.25] $ and the uncertainty space is $\mathcal{W} = \mathcal{W}_1 \times \mathcal{W}_2$ where
\begin{align*}
    \mathcal{W}_1 &= 0.5 \times \{ 0.0, 0.2, 0.4, 0.6, 1.0  \}, \\
    \mathcal{W}_2 &= 2\pi \times \{ 0.0, 0.1, 0.125, 0.2, 0.25, 0.3, 0.375, 0.45, 0.5, \\
    & ~~~~~ 0.575, 0.625, 0.7, 0.75, 0.875, 0.95, 1.0 \}.
\end{align*}
The global robust solution $\bs{x}^\star$ is approximately $[-0.178, 0.289]^\top$ while the corresponding optimal objective $\max_{\bs{w} \in \mathcal{W}} g(\bs{x}^\star, \bs{w})$ is approximately $-4.2$. The nominal uncertainty value is assumed to be $\bar{\bs{w}} = [0,0]^\top$. 

\subsection{Cliff}

The Cliff test problem is a function network consisting of $K=6$ nodes. The first five are black-box (unknown) while the final one is a white-box (known) summation function. The problem involves five design variables $\bs{x}=[ x_1,x_2,x_3,x_4,x_5 ]^\top$ and five uncertainties $\bs{w}=[w_1, w_2, w_3, w_4, w_5]^\top$. The ground-truth equations can be expressed in the form of \eqref{eq:network-functions} as follows
\begin{align*}
    h_k &= \frac{-10}{1 + 0.3 \exp(6x_i + 3\sin(w_i))} - 0.2 (x_i + 0.5 \sin(w_i))^2, ~~ k = 1,\ldots,5, \\
    h_6 &= \sum_{k=1}^5 h_k.
\end{align*}
The objective $g(\bs{x},\bs{w})$ is given by \eqref{eq:network-objective} with $\bs{c} = [0,0,0,0,0,1]^\top$. The design space is $\mathcal{X} = [0, 5]^5$ and the uncertainty space is $\mathcal{W} = W^5$ where $W = \pi \times \{ 0, 0.5, 1 \} - \pi/2$.
The global robust solution $\bs{x}^\star$ is $[1.2, 1.2, 1.2, 1.2, 1.2]^\top$ while the corresponding optimal objective $\max_{\bs{w} \in \mathcal{W}} g(\bs{x}^\star, \bs{w})$ is approximately $-2.9$. The nominal uncertainty value is assumed to be $\bar{\bs{w}} = [0,0,0,0,0]^\top$. 

\subsection{Rosenbrock}

The Rosenbrock problem is a function network consisting of $K=4$ nodes. The first three are are black-box (unknown) while the final one is a white-box (known) summation function. The problem involves one design variable $x$ and two uncertainties $\bs{w} = [w_1,w_2]^\top$. The ground-truth equations can be expressed in the form of \eqref{eq:network-functions} as follows
\begin{align*}
    h_1 &= (x + w_1)^2, \\
    h_2 &= (x + w_1 - 1)^2, \\
    h_3 &= (w_2 - h_1)^2, \\
    h_4 &= -100h_3 - h_2.
\end{align*}
The objective $g(\bs{x},\bs{w})$ is given by \eqref{eq:network-objective} with $\bs{c} = [0,0,0,1]^\top$. The design space is $\mathcal{X} = [-1, 2]$ and the uncertainty space is $\mathcal{W} = \mathcal{W}_1 \times \mathcal{W}_2$ where
\begin{align*}
    \mathcal{W}_1 &= 0.2 \times \{ \textstyle\frac{p}{19} : p = 0, \ldots, 19 \} - 0.1, \\
    \mathcal{W}_2 &= \{ -0.1, 0, 0.1 \}.
\end{align*}
The global robust solution $\bs{x}^\star$ is $1$ while the corresponding optimal objective $\max_{\bs{w} \in \mathcal{W}} g(\bs{x}^\star, \bs{w})$ is approximately $-143.8$. The nominal uncertainty value is assumed to be $\bar{\bs{w}} = [0,1.4]^\top$. 

\subsection{Modified Sine}

The Modified Sine problem is a function network consisting of $K=7$ nodes. The first six are are black-box (unknown) while the final one is a white-box (known) summation function. The problem involves two design variables $\bs{x}=[ x_1,x_2 ]^\top$ and two uncertainties $\bs{w}=[w_1, w_2]^\top$. The ground-truth equations can be expressed in the form of \eqref{eq:network-functions} as follows
\begin{align*}
    h_1 &= x_1 + w_1, \\
    h_2 &= -\sin(2\pi h_1^2), \\
    h_3 &= -h_1^2 - 0.2h_1, \\
    h_4 &= x_2 + w_2, \\
    h_5 &= -\sin(2\pi h_4^2), \\
    h_6 &= -h_4^2 - 0.2h_4, \\
    h_7 &= h_2 + h_3 + h_5 + h_6.
\end{align*}
The objective $g(\bs{x},\bs{w})$ is given by \eqref{eq:network-objective} with $\bs{c} = [0,0,0,0,0,0,1]^\top$. The design space is $\mathcal{X} = [-1, 1]^2$ and the uncertainty space is $\mathcal{W} = W^2$ where
\begin{align*}
    W &= 0.5 \times \{ 0, 0.33, 0.50, 0.66, 1 \} - 0.25.
\end{align*}
The global robust solution $\bs{x}^\star$ is $[0, 0]^\top$ while the corresponding optimal objective $\max_{\bs{w} \in \mathcal{W}} g(\bs{x}^\star, \bs{w})$ is approximately $-0.891$. The nominal uncertainty value is assumed to be $\bar{\bs{w}} = [0, 0]^\top$. 

\section{Comparison to surrogate-free robust optimization} \label{app:surrogate-free-comparison}

\setcounter{figure}{0}
\renewcommand{\thefigure}{D\arabic{figure}}

Here, we compare BONSAI to a class of robust optimization methods that do not rely on surrogate models and take inspiration from evolutionary or local search heuristics. While these methods can be effective in certain cases, they typically suffer from some form of sample inefficiency. To highlight this distinction, we compare BONSAI to the method proposed by Bertsimas, Nohadani, and Teo (BNT) \cite{Bertsimas_unconstrained}. We also provide a brief discussion of related evolutionary approaches and how they differ from our method in both philosophy and performance characteristics.

\subsection{The BNT method for robust black-box optimization}

The BNT method was developed for solving robust (max-min) optimization problems in simulation-based settings. At each iteration, the method performs the following steps:
\begin{enumerate}
    \item \textbf{Neighborhood exploration:} Given a current decision variable $\bs{x}_k$, the method explores the local neighborhood to identify adversarial perturbations that maximize the objective $g(\bs{x}_k, \bs{w})$. These worst-case $\bs{w}$ values are obtained by running short gradient descent routines (or local searches) over the uncertainty space.
    \item \textbf{Robust update direction:} Once a set of adversarial directions has been collected, the method formulates a second-order cone program (SOCP) to compute a (robust) ascent direction for $\bs{x}$ that pushes the design away from all known adversarial points.
    \item \textbf{Step size selection:} The step length is chosen to minimally move $\bs{x}_k$ while ensuring robustness against the discovered $\bs{w}$ samples. The process is repeated until the SOCP becomes infeasible, indicating convergence to a locally robust solution.
\end{enumerate}

This method is inherently local and does not reuse information across iterations. It is also sample-intensive, requiring many objective evaluations and gradient estimates to identify the worst-case $\bs{w}$ samples at every iteration. Despite this, the algorithm provides a principled approach to simulation-based robust design and serves as a useful benchmark for comparison.

\subsection{Experimental setup}

To evaluate performance, we implement the BNT algorithm and apply it to the 2D polynomial case study described in Section~\ref{sec:polynomial_case}. At each iteration, we restrict the neighborhood exploration step to a maximum of 10 function evaluations, consistent with the approach in the original paper. We run 30 independent trials, each initialized from a different random starting point. The recommended design point is taken as the current iterate. 
For BONSAI, we use the same setup and acquisition procedure described in the main text, with evaluation budget matched to that used in the BNT runs. Both methods are compared based on the best worst-case objective value attained after a given number of total function evaluations. 
% Figures~\ref{fig:bertsimas_ascent_steps} and \ref{fig:bonsai_vs_bertsimas} summarize the performance across runs.

\subsection{Results and discussion}

Figure~\ref{fig:bertsimas_ascent_steps} shows the trajectory of four representative BNT runs. Depending on the initial point, the method either converges to a robust local solution or reaches the global robust optimum. Since the method only uses local information and does not construct a global model, its behavior can vary widely across runs. This kind of sensitivity is a common drawback of local search–based methods in the presence of multiple (locally) robust solutions.

\begin{figure}[tb]
  \centering
  \includegraphics[width=0.9\textwidth]{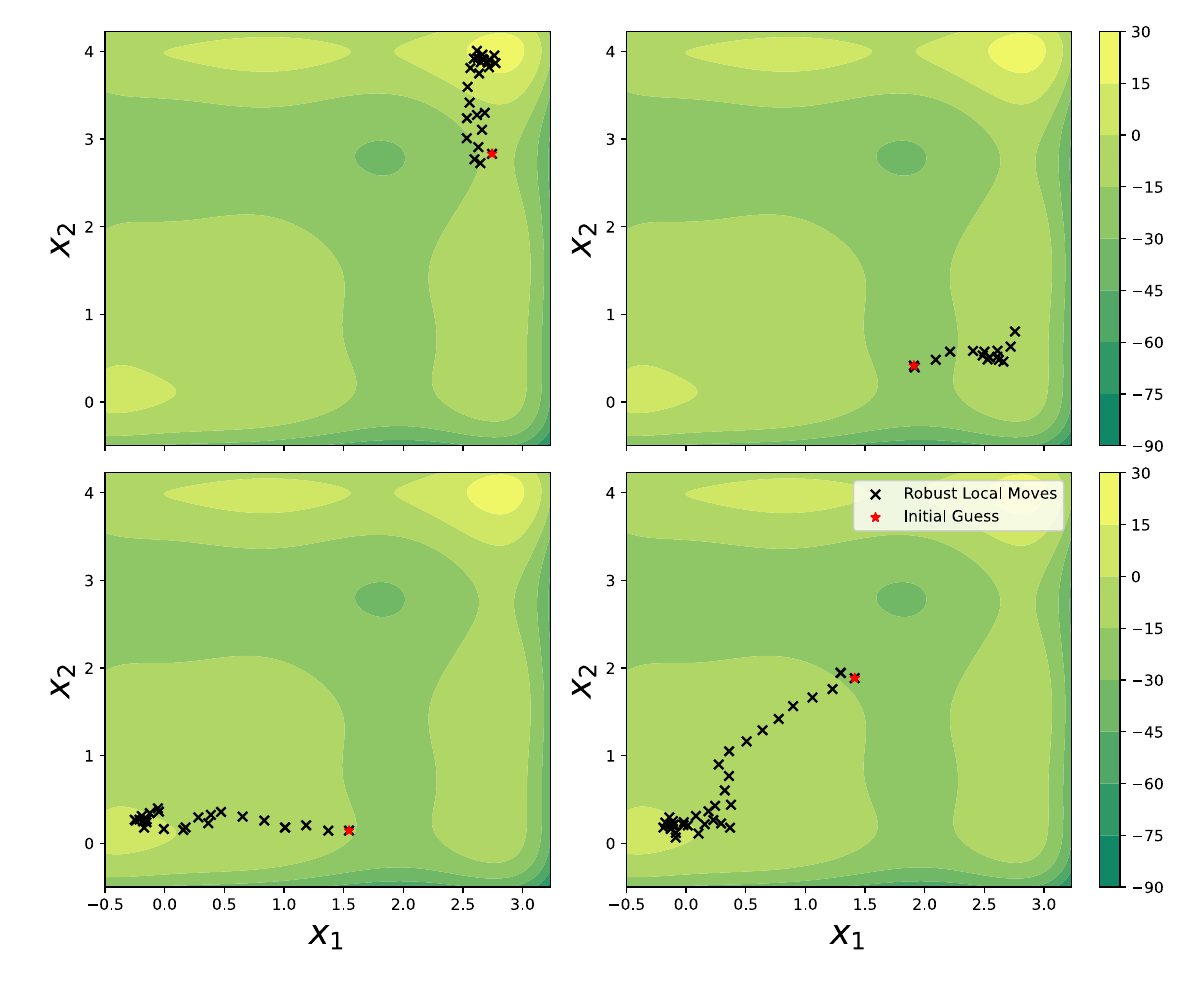}
  \caption{Representative runs of the BNT local robust optimization method on the 2D polynomial case study. \textbf{Top:} Certain initializations lead to convergence to local (suboptimal) robust solutions. \textbf{Bottom:} Other initializations recover the global robust optimum. Red star denotes the initial point; black ‘x’ markers indicate the iterates; background shading shows the nominal objective surface.}
  \label{fig:bertsimas_ascent_steps}
\end{figure}

In contrast, BONSAI uses a global surrogate that models $g(\bs{x}, \bs{w})$ across the joint design–uncertainty space, enabling more principled and data-efficient exploration. Figure~\ref{fig:bonsai_vs_bertsimas} shows the best worst-case value attained as a function of the number of evaluations. BONSAI converges rapidly to the global robust solution with low variance across runs, while the BNT method exhibits slower progress and significantly more variability. This reflects the fundamental tradeoff: without a surrogate, BNT spends many evaluations rediscovering adversarial points at each iteration, whereas BONSAI integrates prior observations to more effectively guide sampling. Note that BONSAI performs poorly in the early phases when the surrogate model is trained on too little data and thus recommends inaccurate designs. This quickly gets corrected as new data comes in and the surrogate can overcome the prior, achieving a significant and steady boost in performance after roughly 25 function evaluations.

\begin{figure}[tb]
  \centering
  \includegraphics[width=0.65\textwidth]{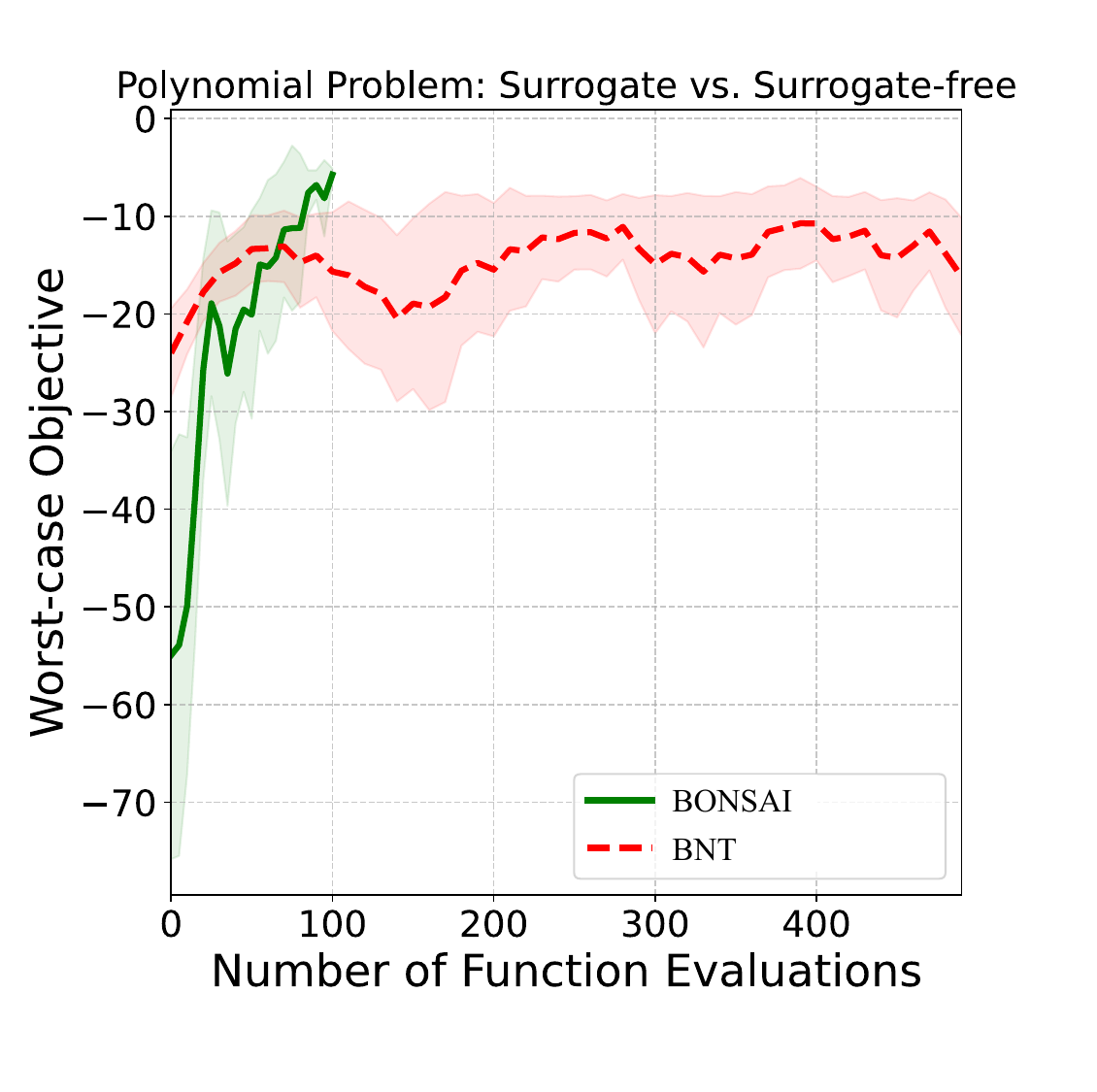}
  \caption{Recommended worst-case objective value as a function of the number of function evaluations for BONSAI (green) and the BNT robust local search method (red) for the polynomial case study. BONSAI consistently finds better solutions with fewer evaluations and lower variance after an initial warm-up period.}
  \label{fig:bonsai_vs_bertsimas}
\end{figure}

We expect similar limitations to hold for other surrogate-free robust optimization methods including evolutionary strategies such as dual-stage differential evolution (DE) \cite{du2024novel} or robust CMA-ES \cite{kruisselbrink2011using}. These methods typically either co-evolve a population of designs with a population of uncertainties/adversaries or approximate worst-case rankings inside a single population. While such approaches can be competitive given large evaluation budgets, they are rarely sample-efficient and often treat the objective as a pure black box without exploiting structure or utilizing any rigorous form of uncertainty quantification.
In contrast, BONSAI explicitly models structure via function networks and uses acquisition strategies that prioritize informative samples. A more comprehensive exploration of how these different classes of methods compare on a wider variety of problems (including high-dimensional and noisy settings) is an important area worthy of future work. 

\section{Ablation Studies}
\label{app:ablation}

\setcounter{figure}{0}
\renewcommand{\thefigure}{E\arabic{figure}}

In this section, we conduct two ablation studies to analyze the contribution of BONSAI's two primary components: its structure-aware acquisition strategy and its function network-based recommendation rule.

\subsection{Search strategy versus recommendation rule}

In this experiment, we apply nine variants combining three search strategies (BONSAI, ARBO, and Random) with three recommendation rules: \textbf{GPFN}, uses the posterior mean of the full function network; \textbf{GP-Quantile}, applies a lower confidence bound (LCB) $\mu_{n}(\bs{x}, \bs{w}) + \beta^{1/2} \sigma_n (\bs{x}, \bs{w})$ with fixed width $\beta^{1/2} = 2$ to a standard GP model; and \textbf{GP}, a special case of GP-Quantile with $\beta=0$ (posterior mean only). 

Figure~\ref{fig:ablation-search-recommendation} shows the results on the Polynomial function network problem. The BONSAI search strategy is clearly the most effective in exploring the design space. For recommendation, methods leveraging function network structure (GPFN) substantially outperform those relying on black-box GP models. These results support our central claim: leveraging intermediate structure during both search and recommendation is critical for robust design.

\begin{figure}[tb]
  \centering
  \includegraphics[width=0.7\textwidth]{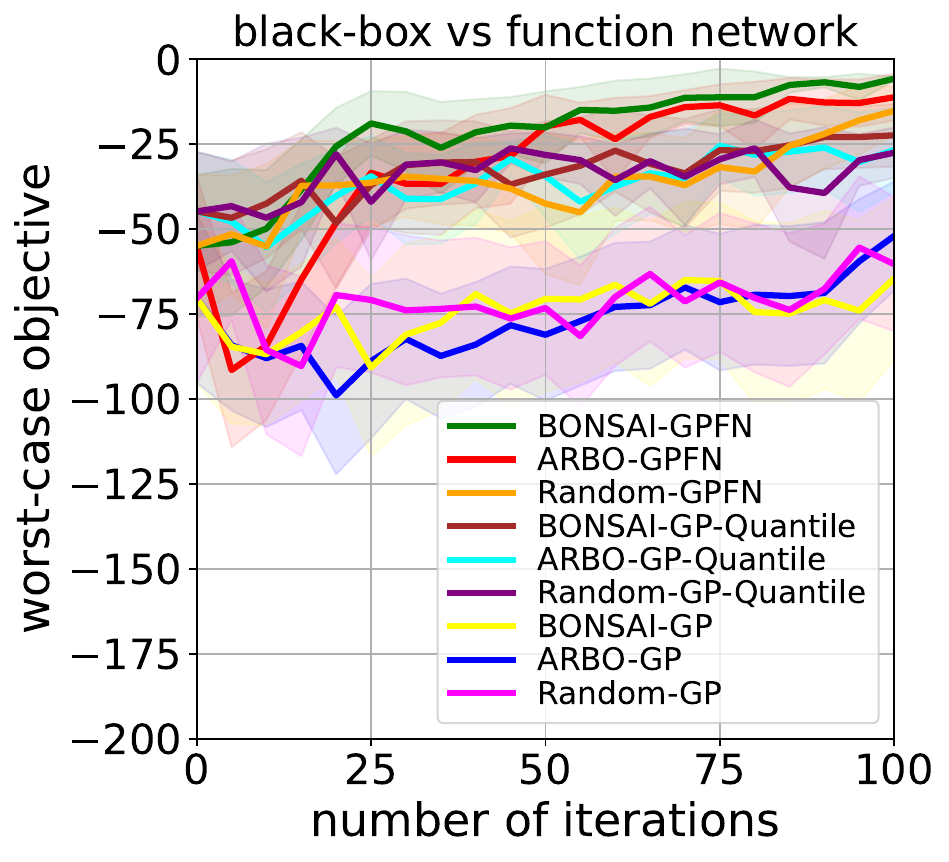}
  \caption{Ablation study on the Polynomial test problem. We compare nine algorithmic variants formed by pairing three search strategies with three recommendation rules. BONSAI–GPFN consistently achieves the best performance, highlighting the importance of both structure-aware exploration and recommendation.}
  \label{fig:ablation-search-recommendation}
\end{figure}

\subsection{Quantile-based recommendation using GP function network}

To assess whether risk-averse recommendations could further improve performance, we investigate a quantile-based recommendation rule derived from the full GPFN posterior. Since closed-form confidence bounds are not available for general function networks, we approximate the 5\% quantile using 256 independently drawn posterior function samples. This rule, denoted \textbf{GPFN-Quantile}, is evaluated every 25 iterations for tractability.
Figure~\ref{fig:GPFN-vs-Quantile} compares GPFN-Quantile to standard GPFN. In early iterations, the quantile rule leads to more robust designs by guarding against model error. As data accumulate, both rules converge, suggesting that model accuracy ultimately dominates recommendation performance.

\begin{figure}[tb]
  \centering
  \includegraphics[width=0.7\textwidth]{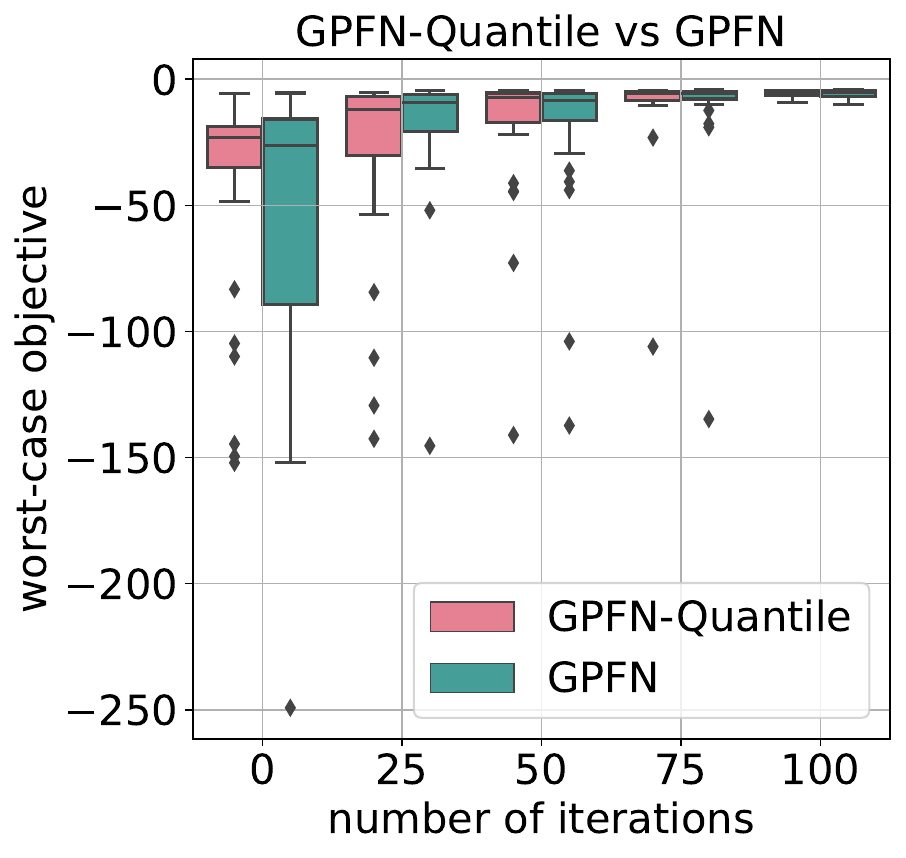}
  \caption{Comparison of GPFN and GPFN-Quantile recommendation strategies. GPFN-Quantile uses a 5\% pessimistic quantile based on posterior function samples (same as that used in the Thompson sampling procedure). While GPFN-Quantile provides better performance in early iterations, both approaches converge as model fidelity improves.}
  \label{fig:GPFN-vs-Quantile}
\end{figure}

\section{Average Runtime per Iteration}
\label{app:runtime}

\setcounter{figure}{0}
\renewcommand{\thefigure}{F\arabic{figure}}

\setcounter{table}{0}
\renewcommand{\thetable}{F\arabic{table}}

We report the average runtime per iteration for acquisition function optimization in the BONSAI and ARBO methods. These values are computed by dividing the total optimization time over 100 iterations and averaging across 30 independent runs for each case study. 
% Since each method was run for 100 iterations, the total runtime can be recovered by multiplying the reported values by 100. 
These results provide a normalized measure of computational cost that facilitates fair comparisons across problems and between methods. All timings were recorded on a system with an Intel(R) Core(TM) i7-10700K CPU.

As expected, BONSAI incurs higher computational overhead per iteration than ARBO (see Table \ref{tab:runtime-comparison}). This additional cost stems from the use of the structured surrogate models, which require propagating samples through the function network and solving a more complex acquisition function optimization problem. However, the intended use case for BONSAI is optimization of expensive black-box systems (such as high-fidelity simulations or physical experiments), where each function evaluation may take minutes, hours, or longer. In such settings, the cost of acquisition function optimization is typically negligible relative to the cost of evaluating the true system. 
As such, we view the current implementation as a conservative, structure-aware baseline for robust optimization. Reducing the computational overhead of BONSAI (particularly for lower-cost applications) remains an important and interesting direction for future work. For example, it may be possible to develop adaptive versions of BONSAI that dynamically balance acquisition cost with evaluation cost, enabling more efficient decision-making in resource-constrained settings.

\begin{table}[tb]
\centering
\caption{Average runtime per iteration (in seconds) for BONSAI and ARBO. Reported values represent the average over 30 independent runs.}
\label{tab:runtime-comparison}
\begin{tabular}{lcc}
\toprule
\textbf{Case Study} & \textbf{BONSAI (sec/iter)} & \textbf{ARBO (sec/iter)} \\
\midrule
Polynomial             & 1.14    & 1.11 \\
Cliff                  & 41.47   & 12.78 \\
Rosenbrock             & 13.94   & 2.37 \\
Modified Sine          & 3.87    & 0.56 \\
Robot Pushing          & 2.47    & 1.55 \\
Heat Exchanger Network & 3.55    & 0.73 \\
Digits Classifier      & 1.60    & 0.97 \\
Vibration Absorber     & 4.06    & 0.66 \\
\bottomrule
\end{tabular}
\end{table}

These results are meant to give a practical sense of BONSAI's computational cost in representative cases rather than provide a full scaling analysis. The most important factors that are expected to affect its runtime are:
\begin{itemize}
    \item \textbf{Function network size:} Each learned node maintains its own probabilistic surrogate model, and these are updated separately. Runtime grows with the number of nodes, although updates can be parallelized and lightweight surrogate models (e.g., sparse or approximate GPs) could help in larger networks.
    \item \textbf{Design and uncertainty dimensions:} Higher-dimensional problems increase the difficulty of the acquisition optimization problems. Although warm starting techniques can help manage this, they could introduce additional suboptimality and still suffer from some cost increase with increasing problem dimension. 
    \item \textbf{Implementation details:} The current codebase is a prototype built for flexibility, not speed. Performance could be improved through batching, caching, and parallelization.
\end{itemize}
A full scaling study across network size and input dimensionality is left to future work.

\end{document}